\documentclass[12pt]{article}

\usepackage[inner=1in,outer=1in,top=1in,bottom=1in,
marginparwidth=.7in,marginparsep=.1in]{geometry}

\usepackage{graphicx}
\usepackage{enumitem}
\setlist[enumerate]{leftmargin=.5in}
\setlist[itemize]{leftmargin=.5in}

\usepackage{algorithmic}

\usepackage[ruled,linesnumbered]{algorithm2e}
\SetKwComment{Comment}{/* }{ */}

\usepackage{amsmath,amssymb,amsthm,fullpage,longtable}
\usepackage{xcolor}
\usepackage{mathabx}
\usepackage{accents}
\usepackage{comment}

\usepackage[
    colorlinks=true,    
    linkcolor=blue,     
    urlcolor=blue,     
    citecolor=blue      
]{hyperref}

\def\R{\mathbb{R}}

\def\C{\mathbb{C}}

\def\der{\mathrm{d}}
\def\x{\mathbf{x}}

\def\O{\mathsf{O}}
\def\SO{\mathsf{SO}}

\def\E{\mathbb{E}}

\def\der{\mathrm{d}}

\newcommand{\T}{\mathsf{T}}
\def\1{\mathbf{1}}
\def\E{\mathbb{E}}
\def\ii{\mathrm{i}}

\DeclareMathOperator{\diag}{diag}

\DeclareMathOperator{\Tr}{Tr}
\DeclareMathOperator*{\argmin}{arg\,min}

\newcommand{\bigO}{\mathit{O}}

\newtheorem{theorem}{Theorem}[section]
\newtheorem{lemma}[theorem]{Lemma}
\newtheorem{proposition}[theorem]{Proposition}

\theoremstyle{definition}
\newtheorem{remark}[theorem]{Remark}
\newtheorem*{remark*}{Remark}
\newtheorem*{example*}{Example}

\newtheorem{assumption}[theorem]{Assumption}
\newtheorem{problem}[theorem]{Problem}
\numberwithin{equation}{section}
\numberwithin{figure}{section}

\title{Two datasets are better than one: Method of double moments for 3-D reconstruction in cryo-EM}

\author{
Joe Kileel \thanks{\href{mailto:jkileel@math.utexas.edu}{jkileel@math.utexas.edu}}\\
\small Department of Mathematics and Oden Institute, University of Texas at Austin, USA
\and
Oscar Mickelin \thanks{\href{mailto:oscarmi@tsinghua.edu.cn}{oscarmi@tsinghua.edu.cn}} \\
\small Yau Mathematical Sciences Center, Tsinghua University, China
\and
Amit Singer \thanks{\href{mailto:amits@math.princeton.edu}{amits@math.princeton.edu}} \\ 
\small Program in Applied and Computational Mathematics and Department of Mathematics,\\
\small Princeton University, USA
\and
Sheng Xu 
\thanks{Corresponding author. \href{mailto:sxu21@princeton.edu}{sxu21@princeton.edu}}\\
\small Program in Applied and Computational Mathematics, Princeton University, USA
}

\usepackage{hyperref}

\begin{document}

\maketitle

\begin{abstract}
Cryo-electron microscopy (cryo-EM) is a powerful imaging technique for reconstructing three-dimensional molecular structures from noisy tomographic projection images of randomly oriented particles. We introduce a new data fusion framework, termed the method of double moments (MoDM), which reconstructs molecular structures from two instances of the second-order moment of projection images obtained under distinct orientation distributions: one uniform, the other non-uniform and unknown. We prove that these moments generically uniquely determine the underlying structure, up to a global rotation and reflection, and we develop a convex-relaxation-based algorithm that achieves accurate recovery using only second-order statistics. Our results demonstrate the advantage of collecting and modeling multiple datasets under different experimental conditions, illustrating that leveraging dataset diversity can substantially enhance reconstruction quality in computational imaging tasks.
\end{abstract}

\let\thefootnote\relax

\section{Introduction}\label{sec:introduction}

Cryo-electron microscopy has become an increasingly popular technique for single-particle three-dimensional (3-D) structure reconstruction \cite{frank2006three,kuhlbrandt2014resolution,bai2015cryo,callaway2020revolutionary}. It aims to recover the underlying molecular structure from a large collection of noisy two-dimensional (2-D) tomographic projection images taken at unknown and typically random orientations, with applications in structural biology, medicine, and drug discovery \cite{renaud2018cryo,van2020cryo,duan2024cryo,shi2023cryo,wrapp2020cryo}. 

A fundamental challenge in cryo-EM is the high noise level and the unknown random viewing directions of particles in the sample, which complicate the reconstruction process. This paper introduces a new statistical approach for 3-D reconstruction that relies solely on low-order statistics, specifically, two instances of the second-order moment of the observed projection images, each computed from images with different distributions of orientations. Crucially, we prove that under mild conditions, these moments uniquely determine the underlying structure (up to a global rotation and reflection), and we devise a computationally efficient algorithm based on convex relaxation to perform the reconstruction.

The sample complexity of our method scales as $\omega(\mathrm{SNR}^{-2})$, where $\mathrm{SNR}$ denotes the signal-to-noise ratio of the input images. This improves upon previous method-of-moments approaches that either require access to higher order moments, typically resulting in sample complexity scaling at least as $\omega(\mathrm{SNR}^{-3})$, or do not offer uniqueness guarantees. Furthermore, in contrast to earlier methods that often suffer from spurious local minima and stagnation, our algorithm demonstrates robust performance in numerical experiments. See Section~\ref{sec:existing_methods} for a detailed comparison with related work.

An important practical motivation for our method comes from the fact that projection orientations in cryo-EM are typically not uniformly distributed. This anisotropy in orientation distribution can result from experimental factors such as preferred particle orientations or sample preparation artifacts \cite{taylor2008retrospective, liu2013deformed, glaeser2017opinion, noble2018routine, carragher2019current, lyumkis2019challenges, baldwin2020non}. To address this issue, recent experimental advances have sought to manipulate particle orientation distributions to be more uniform, by, for example, tilting the specimen \cite{aiyer2024overcoming}, applying laser flash melting \cite{straub2025laser}, adding charged detergents \cite{li2021effect}, introducing ultrasonic excitation \cite{williams2025overcoming}, encapsulating molecules in liposomes \cite{zikai2025liposome} or in protein shells (``nanocrates'') \cite{jenkins2025overcoming}, or by high-speed droplet vitrification before protein diffusion \cite{gusach2025outrunning}.

Our proposed method, termed the method of double moments (MoDM), is naturally suited to such experimental scenarios. It leverages two datasets collected under distinct orientation distributions: one ideally uniform, made possible through experimental interventions, and the other non-uniform and unknown \textit{a priori}. Although this setting is increasingly relevant in practice, it has not been systematically explored in existing statistical reconstruction frameworks. Our results demonstrate that fusing such complementary datasets enables exact structure recovery using only second-order statistics, thereby opening new avenues for robust and efficient cryo-EM reconstruction despite variations in particle orientations.

While this advantage is illustrated here within the method-of-moments framework, the underlying principle is more general. We anticipate that maximum likelihood and Bayesian inference approaches to cryo-EM reconstruction can also benefit from specifying separate orientation-distribution priors for each dataset, rather than merging all images into a single dataset under a prior that is typically assumed to be isotropic. Moreover, this concept may extend beyond cryo-EM to other imaging modalities, such as X-ray Free Electron Lasers (XFEL) \cite{von2018structure, zhao2024structure}, or even to multimodal reconstruction scenarios, for example, those that combine anisotropic cryo-EM with isotropic XFEL data, or integrate cryo-EM with Small-Angle X-ray Scattering (SAXS) data \cite{kim2017cross,lyu2019cross}.

The remainder of this paper is organized as follows.
Section~\ref{sec:notation} introduces the notation used throughout the paper.
Section~\ref{sec:image_formation} describes the image formation model considered in cryo-EM, and Section~\ref{sec:existing_methods} reviews related work.
Section~\ref{sec:setup} details the assumptions and conventions adopted in this study and precisely states the main reconstruction problem.
Section~\ref{sec:algorithm} introduces the proposed algorithm for solving the problem.
In particular, Section~\ref{sec:numerics} provides numerical experiments showing the effectiveness of the method.
Section~\ref{sec:uniqueness} proves that the problem admits a unique solution. Section~\ref{sec:stability} establishes the stability of the corresponding
Cholesky decomposition under potentially non-uniform viewing distributions of the first dataset.
Finally, Section~\ref{sec:conclusion} concludes the paper.
The code implementing the proposed algorithm is available at \href{https://github.com/oscarmickelin/modm}{https://github.com/oscarmickelin/modm}.

\subsection{Notation}\label{sec:notation}

Denote by $\R$ the set of real numbers, by $\mathbb{Z}_{\ge 0}$ the set of nonnegative integers, and by $\1_{\mathcal{X}}$ the indicator function of a set $\mathcal{X}$. 
For a matrix $M$, denote by $M^{\T}$ and $M^{\mathsf{H}}$ its transpose and conjugate transpose, respectively. 
We denote by $\O(n)$ the set of real orthogonal matrices of dimension $n$, that is,
$\O(n) = \{ M \in \mathbb{R}^{n \times n} : M M^{\T} = M^{\T} M = I \}$.
Similarly, $\SO(n)$ denotes the special orthogonal group of dimension $n$, consisting of real orthogonal matrices with determinant $+1$, i.e.,
$\SO(n) = \{ M \in \O(n) : \det(M) = 1 \}$.
We denote by $\mathsf{U}(n)$ the set of unitary matrices of dimension $n$, that is, $\mathsf{U}(n)=\{M\in\mathbb{C}^{n\times n}:MM^{\mathsf{H}}=M^{\mathsf H}M=I\}$. Moreover, we use $\O(n,\mathbb{C})$ to denote the set of complex orthogonal matrices of dimension $n$, defined by
$\O(n,\mathbb{C}) = \{ M \in \mathbb{C}^{n \times n} : M M^{\T} = M^{\T} M = I \}$. Note that, unlike unitary matrices, the orthogonality condition here does not involve complex conjugation.

For matrices $M_1 \in \mathbb{C}^{n_1 \times n_1}, \ldots, M_N \in \mathbb{C}^{n_N \times n_N}$, 
let $\mathrm{blockdiag}_{i=1,\ldots,N}(M_i)$ denote the block-diagonal matrix of size 
$\big(\sum_{i=1}^N n_i\big) \times \big(\sum_{i=1}^N n_i\big)$ defined as
\begin{align*}
\text{blockdiag}_{i=1, \ldots,  N}(M_i) := \begin{bmatrix}
M_1 & 0  & \cdots & 0 \\
0 & M_2  & \cdots & 0 \\
\vdots & \vdots & \ddots & \vdots \\
0 & 0 & \cdots &  M_N
\end{bmatrix}.    
\end{align*}
Equivalently, for two matrices $M_1$ and $M_2$, we write $M_1 \oplus M_2$ to denote their block-diagonal concatenation:
\[
M_1 \oplus M_2 :=
\begin{bmatrix}
M_1 & 0 \\
0 & M_2
\end{bmatrix}.
\]
This binary operation can be extended recursively to $N$ matrices as
\[
M_1 \oplus M_2 \oplus \cdots \oplus M_N := \mathrm{blockdiag}_{i=1,\ldots,N}(M_i).
\]
For integrable functions $f, g : \mathbb{R}^d \to \mathbb{C}$, the convolution of $f$ and $g$ is defined as
\[
(f * g)(\mathbf{x})
:= \int_{\mathbb{R}^d} f(\mathbf{x} - \mathbf{y}) \, g(\mathbf{y}) \, \der \mathbf{y},
\qquad \mathbf{x} \in \mathbb{R}^d.
\]
For a square-integrable function $f:\R^d \to \R$, we use the following convention for the Fourier transform:
\[
\widehat{f}(\boldsymbol{\omega})
= \int_{\mathbb{R}^d} f(\mathbf{x}) \, e^{-2\pi i \, \mathbf{x}\cdot \boldsymbol{\omega}} \, \der \mathbf{x}, \qquad \boldsymbol{\omega} \in \mathbb{R}^d,
\]
where $\mathbf{x}\cdot \boldsymbol{\omega} = \sum_{j=1}^d x_j \omega_j$ denotes the Euclidean inner product. For a square-integrable function $f:\SO(3)\to\mathbb{C}$, we define the $L_2$ norm by
\[
\|f\|_{L_2(\SO(3))}
=
\Big(
\int_{\SO(3)} |f(R)|^2 \, \der R
\Big)^{1/2},
\]
where $\der R$ denotes the Haar measure on $\SO(3)$ satisfying
\[
\int_{\SO(3)} \der R = 1.
\]

\subsection{Image formation model} \label{sec:image_formation}
We denote the electrostatic potential of a target 3-D molecular structure by $\Phi^*:\mathbb{R}^3 \rightarrow \mathbb{R}$, and represent a random rotation by $R \in \SO(3)$. The action of $R$ on $\Phi^*$ is written as $R^\T \cdot \Phi^* : \mathbb{R}^3 \rightarrow \mathbb{R}$, and defined by
\begin{equation} \label{eq:def-rotated-3-D}
   (R^\T\cdot \Phi^*) (x_1,x_2,x_3) :=  \Phi^*\left(R(x_1,x_2,x_3)\right),~~~\text{for}~(x_1,x_2,x_3)\in\R^3,
\end{equation}
where $R^\T \cdot \Phi^*$ denotes the rotated potential via a change of coordinates by the rotation $R$. The corresponding tomographic projection image $I_R(x_1,x_2)$, acquired under viewing direction $R$, is modeled as
\begin{equation}\label{eq:def_proj_images}
I_R(x_1,x_2) = \int_{-\infty}^{\infty} (R^\T\cdot \Phi^*)(x_1,x_2,x_3)\mathrm{d} x_3 + \varepsilon(x_1,x_2),~~~\text{for}~(x_1,x_2)\in\R^2,
\end{equation}
where $\varepsilon$ is an additive Gaussian white noise term with variance $\sigma^2 I$, independent of the signal. The noise level $\sigma^2$ can typically be estimated in advance from the observed images. We further assume the random rotation $R$ follows an unknown probability distribution $\rho^*$ over the rotation group $\SO(3)$. In a standard cryo-EM experiment, one observes a collection of noisy 2-D projection images $I_1, \ldots, I_N$, where each image $I_i$ is associated with an independent unknown rotation $R_i\sim\rho^*$. For simplicity, we assume throughout the paper that all projection images are perfectly centered and ignore the effects of in-plane shifts and the contrast transfer function. In practice, both factors should be accounted for in a complete image formation model; however, neglecting them allows us to focus on the core statistical and geometric aspects of our analysis.  Further, there are established methods to account for these effects (as explained below). 

To better understand the relationship between the projections and the 3-D structure, it is useful to consider the 2-D Fourier transform of the projection images. By the Fourier slice theorem \cite{natterer2001mathematics}, the 2-D Fourier transform $\widehat{I_R}$ of a projection image $I_R$ corresponds to a central planar slice of the 3-D Fourier transform $\widehat{\Phi^*}$ of the volume, taken perpendicular to the viewing orientation $R$. Specifically,
taking the Fourier transform of \eqref{eq:def_proj_images} yields
\begin{align}\label{eq:Fourier_slice}
\widehat{I_R}(\omega_1,\omega_2) = (R^\T \cdot \widehat{\Phi^*})(\omega_1,\omega_2,0)+\widehat\varepsilon(\omega_1,\omega_2),~~~\text{for}~(\omega_1,\omega_2)\in\R^2.
\end{align}
This frequency-domain formulation plays a central role in many reconstruction algorithms and is particularly useful for statistical modeling of image formation under varying orientations.

The  goal in cryo-EM then is to recover the structure $\Phi^*$ from the noisy projection images, without prior knowledge of the distribution $\rho^*$ (see also Section \ref{sec:precise}). Due to the high noise level in individual images, achieving meaningful resolution in the reconstructed 3-D structure typically requires tens of thousands of projection images or more \cite{bai2015cryo,wrapp2020cryo}. 

A more refined and realistic image formation model than \eqref{eq:def_proj_images} also accounts for both in-plane shifts of the projection images and optical aberrations introduced by the microscope.  The latter is modeled via convolution with a point spread function $h_i(x_1,x_2)$ (see Section \ref{sec:notation} for the definition of convolution) whose Fourier transform is known as the contrast transfer function \cite{singer2018mathematics}. Each $h_i$ is an approximately radially symmetric, highly oscillatory function with frequent zero crossings, which complicates the inversion process. Under this model, the projection image $I_i$ corresponding to rotation $R_i$ and in-plane shift $\mathbf{t}_i$ is given by 
\begin{align}\label{eq:more-complex}
I_i(x_1,x_2) = h_i(x_1,x_2) * S_{\mathbf{t}_i} \Big(\int_{-\infty}^{\infty} (R_i^\T\cdot \Phi^*)(x_1,x_2,x_3)\mathrm{d} x_3\Big) + \varepsilon(x_1,x_2),~~~\text{for}~(x_1,x_2)\in\R^2,
\end{align}
where $S_{\mathbf{t}_i}$ denotes a two-dimensional shift of the image by the vector $\mathbf{t}_i\in\R^2$, which in Fourier-space becomes
\begin{align} \label{eq:more-complex2}
\widehat{I}_i(\omega_1,\omega_2) = \widehat{h}_i(\omega_1,\omega_2) \cdot e^{-2\ii\pi \mathbf{t}_i \cdot \boldsymbol{\omega}} \cdot \Big((R^\T \cdot \widehat{\Phi^*})(\omega_1,\omega_2,0)\Big) + \widehat{\varepsilon}(\omega_1,\omega_2),~~~\text{for}~\boldsymbol{\omega} =(\omega_1,\omega_2)\in\R^2.
\end{align}
In the proposed algorithm, the projection images $\{{I_i}\}_{i=1}^N$ are used solely to estimate the second-order moment of the underlying structure. The influence of the point spread function can be compensated for during this estimation, provided that the functions $h_i$, for $i = 1, \ldots, N$, exhibit sufficiently non-overlapping zero crossings \cite{shi2022ab,marshall2023fast}. Similarly, in-plane shifts can be corrected using established centering algorithms \cite{heimowitz2021centering}. We therefore assume that these corrections have been performed in advance and omit the point spread functions and shifts in the subsequent analysis.

\subsection{Existing methods}\label{sec:existing_methods}
\paragraph{Maximum likelihood-based approaches.}  
Maximum likelihood estimation provides a principled statistical framework for 3-D structure reconstruction. The goal is to estimate the unknown 3-D molecular structure by maximizing the likelihood of observing a given set of 2-D projection images $\{I_i\}_{i=1}^N$, as described by the formation model \eqref{eq:def_proj_images}. Typically, all cryo-EM projection images are amalgamated into a single dataset, regardless of possible differences in their orientation distributions. Under this setting, the structure is estimated by maximizing the marginal likelihood: 
\begin{align}
\mathcal{L}(\Phi\mid \{I_i\}_{i=1}^N)
= \sum_{i=1}^N \log
\int_{\SO(3)\times \R^2}
p(I_i \mid \Phi, R, \mathbf{t})\,
p(R)\,p(\mathbf{t})
\mathrm{d}R \mathrm{d}\mathbf{t}.
\end{align}
Here, $p(I_i\mid\Phi,R,\mathbf{t})$ denotes the likelihood of observing image $I_i$ given a rotation $R$, a 2-D shift $\mathbf{t}$, and the 3-D structure $\Phi$. The terms $p(R)$ and $p(\mathbf{t})$ represent the priors over rotations and translations, respectively. In practice, $p(R)$ is typically assumed to be uniform due to the lack of knowledge about the true rotation distribution $\rho^*$, which may lead to bias or deformation in the reconstructed volume \cite{xu2025misspecified}. This formulation marginalizes out the latent rotation variable $R$ and translation variable $\mathbf{t}$, reflecting the fact that they are unobserved during data acquisition. 

Since this is a classical example of an incomplete data problem, the expectation-maximization (EM) algorithm is commonly employed to maximize the likelihood \cite{dempster1977maximum}. Popular state-of-the-art methods resort to EM-based procedures (commonly referred to as 3-D iterative refinement) which alternate between estimating the posterior distribution over rotations (E-step) and updating the 3-D structure estimate by maximizing the expected log-likelihood (M-step) \cite{sigworth1998maximum, scheres2012relion, punjani2017cryosparc}.

Despite their empirical success, these approaches suffer from non-convexity, and no global convergence guarantees are known \cite{sigworth2015principles, singer2018mathematics}. As a result, the algorithm may converge to a local maximum, particularly if the initialization is not sufficiently close to the ground truth. This is problematic for downstream applications such as drug discovery and design, where reconstruction accuracy and reliability are critical \cite{merk2016breaking, renaud2018cryo}. Moreover, the methods are computationally intensive, requiring access to the entire dataset during each iteration, further limiting scalability.

\paragraph{Method of moments.} An alternative approach is based on the method of moments, using experimental data to compute empirical moments of the Fourier transforms of the projection images, at increasing orders. The $k$th-order empirical moment $\widetilde{m}_k : \mathbb{R}^{2k} \rightarrow \mathbb{C}$ is given by\footnote{ The complex conjugate in the last mode ensures  symmetry properties of the second-order moment; some related work uses different conventions.}
\begin{align}\label{eq:def_moments}
\widetilde{m}_k\left( \boldsymbol{\omega}_1, \ldots, \boldsymbol{\omega}_k  \right) &:= \frac{1}{N} \sum_{i=1}^N \widehat{I}_i(\boldsymbol{\omega}_1) \cdots  \widehat{I}_i(\boldsymbol{\omega}_{k-1})  \overline{\widehat{I}_i(\boldsymbol{\omega}_k)} - B_k(\boldsymbol{\omega}_1, \ldots, \boldsymbol{\omega}_k,\sigma), 
\end{align}
where $\boldsymbol{\omega}_j \in \mathbb{R}^2$, for $j=1, \ldots, k$, indexes a location in the 2-D Fourier domain of the projected images, and should be understood as part of the coordinate system of the moment tensor. Here, $\widehat{I}_i$ denotes the 2-D Fourier transform of the $i$-th projection image. The term $B_k(\boldsymbol{\omega}_1, \ldots, \boldsymbol{\omega}_k, \sigma)$ denotes a debias term that depends on the noise variance of the projection images (which we assume has been estimated from the data). Since $\widetilde{m}_k$ is computed from data and thus implicitly depends on the underlying molecular structure $\Phi^*$ and rotation distribution $\rho^*$, we sometimes write $\widetilde{m}_k[\Phi^*,\rho^*]$ to highlight its dependence on $\Phi^*$ and $\rho^*$. This notation helps clarify later sections where we consider multiple such distributions.

The method-of-moments approach attempts to reconstruct the molecular structure by matching the empirical moments, computed from the data, to the population moments evaluated at candidate parameters $\Phi$ and $\rho$. These population moments, denoted by $m_k[\Phi, \rho]: \mathbb{R}^{2k} \rightarrow \mathbb{C}$, are defined as
\begin{align}\label{eq:def_pop_moment}
   m_k[\Phi, \rho](\boldsymbol{\omega}_1, \ldots, \boldsymbol{\omega}_k) &:=\E_{\varepsilon}\Big[\int_{\SO(3)} \widehat{I}_{R}(\boldsymbol{\omega}_1)  \cdots  \widehat{I}_{R}(\boldsymbol{\omega}_{k-1})  \overline{\widehat{I}_{R}(\boldsymbol{\omega}_k)} \rho(R) \der R\Big]  - B_k(\boldsymbol{\omega}_1, \ldots, \boldsymbol{\omega}_k,\sigma),
\end{align}
where the expectation is taken over the noise term. Choosing a weighted squared loss, the reconstruction problem becomes 
\[
\min_{\Phi, \rho} \sum_{k=1}^d \lambda_k \|\widetilde{m}_k - m_k[\Phi,\rho] \|^2_F,
\]
for some suitably chosen weights $\lambda_k \in \mathbb{R}_{\geq 0}$, where $k=1, \ldots , d$. We often omit the dependence on $\Phi$ and $\rho$ from the notation when it is clear from context. In particular, the first and second-order population moments take the following form:
\begin{align}
m_1(\boldsymbol{\omega}_1) = & \E_{\varepsilon}\Big[\int_{\SO(3)} \widehat{I_R}(\boldsymbol{\omega}_1) \rho(R) \der R\Big],\label{eq:1st_moment_Cart}\\
m_2(\boldsymbol{\omega}_1,\boldsymbol{\omega}_2) = & \E_{\varepsilon}\Big[\int_{\SO(3)} \widehat{I}_{R}(\boldsymbol{\omega}_1) \overline{\widehat{I_R}(\boldsymbol{\omega}_2)} \rho(R) \der R\Big]- \E_{\varepsilon} [\widehat{\varepsilon}(\boldsymbol{\omega}_1) \overline{\widehat{\varepsilon}(\boldsymbol{\omega}_2)}],\label{eq:2nd_moment_Cart}
\end{align}
where the first expression follows from the assumption that the noise is zero-mean, and the second contains a debias term that depends only on the second-order statistics of the noise. These two moments are the primary focus of this work, as they form the basis for both the theoretical analysis and the proposed reconstruction algorithm.

The method of moments offers several advantages over competing approaches, such as maximum likelihood-based methods with 3-D iterative refinement. One key benefit is that empirical moments in \eqref{eq:def_moments} can be computed using only one or two passes over the data, after which the raw dataset no longer needs to be accessed \cite{zhao2016fast,bhamre2016denoising,shi2022ab,marshall2023fast}. This property significantly reduces computational costs for sufficiently large datasets, especially when compared to the repeated access and high iteration count required in likelihood-based refinement pipelines. Another advantage lies in its applicability to small molecular structures (e.g., below 40 kDa), where state-of-the-art software implementations of 3-D iterative refinement encounter challenges \cite{scheres2012relion,punjani2017cryosparc}. Even in these challenging regimes, particle locations can still be detected reliably from micrographs \cite[Figure 10f–h]{vinothkumar2016single}, which allows empirical moments to be formed and the method of moments to be applied, bypassing the limitations of traditional refinement.

Historically, the method of moments was introduced to the cryo-EM setting by Kam \cite{kam1980reconstruction}, who observed numerically that third-order population moments can uniquely determine bandlimited molecular structures (to be precisely defined in Section~\ref{sec:basis_for_structures}), under the assumption that the viewing directions are uniformly distributed over $\SO(3)$. These observations have since been rigorously justified under various technical assumptions \cite{bandeira2023estimation,fan2024maximum,edidin2024orbit}. Despite these strengths, significant challenges remain. Accurate estimation of the $d$th-order moment requires a number of samples that scales as $\omega(\sigma^{2d})$ \cite{bandeira2023estimation,abbe2018estimation}, which becomes prohibitive even at moderate moment orders such as $d=3$, particularly in high-noise settings. Furthermore, when moments are discretized as tensors, both storage and computational costs grow exponentially with the moment order, presenting serious practical limitations. Recent approaches attempt to mitigate these issues by exploiting compressed low-rank tensor formats \cite{hoskins2024subspace}.

To address these complexities, recent efforts have focused on the use of only second-order moments. Through prior work, reconstruction to a limited resolution is possible under certain structural assumptions or with sufficient side information. In detail, it has been shown that the second-order moment determines the molecular structure uniquely if the structure is sparse in either generic bases or as a Gaussian mixture model \cite{bendory2023autocorrelation,bendory2024sample}, if a homologous model is known \cite{bhamre2015orthogonal}, or if two clean projection images are known \cite{levin20183d}. This also enables the construction of pseudo-metrics for comparing cryo-EM images directly, without explicit reconstruction \cite{zhang2024moment}. Moreover, for generic orientation distributions, it has been shown that there exists a finite set of reconstructions consistent with a given second-order moment \cite{sharon2020method}. These studies, along with related works \cite{bhamre2017anisotropic, huang2023orthogonal}, have developed algorithms for \textit{ab initio} reconstruction that can subsequently be refined using iterative methods. However, obtaining high-resolution reconstructions directly from moments remains highly challenging. The reconstruction problem based solely on the second moment is inherently ill-posed, and even with sparsity priors or additional side information, achieving moderate resolutions is computationally demanding. Further difficulties arise from the nonconvex and high-dimensional nature of the associated optimization problem, which often causes iterative methods to stagnate or converge to spurious local minima, thereby failing to recover the ground-truth structure.

\subsection{Contributions}
This paper introduces a new data fusion approach that highlights how ``the whole (of a dataset) is greater than the sum of its parts'' and the central role of orientation distributions for reconstruction. More precisely, our work takes advantage of variations in the orientation distributions between two datasets to achieve increased performance, as compared to simply using one dataset or combining multiple datasets into one. The key insight is to exploit recent experimental advances in cryo-EM sample preparation that enable the collection of multiple datasets of the same underlying molecule, each associated with a different distribution over 3-D orientations.

Specifically, we assume access to empirical second-order moments of the form
\begin{equation}\label{eq:match_double_moments}
\widetilde{m}_2[\Phi^*,\rho_1^*], \qquad  \text{ and } \qquad \widetilde{m}_2[\Phi^*,\rho_2^*],
\end{equation}
where $\rho_1^*$ and $\rho_2^*$ denote two distinct underlying rotation distributions that generate the respective datasets. Throughout this paper, we assume that $\rho_1^*$ is uniform over $\SO(3)$\footnote{We use both $\rho_1$ and $\rho_1^*$ to denote the uniform/Haar distribution over $\SO(3)$.}, 
and $\rho_2^*$ is non-uniform over $\SO(3)$ but \textit{in-plane uniform}.   This means that the distribution of the resulting projection images is invariant to image-plane rotations; see \eqref{eq:def_in_plane_uniform}. Without loss of generality, this can be achieved by augmenting the dataset with randomly rotated copies of the projection images. Additionally, we assume that $\rho_2^*$ is invariant to chirality, meaning that the distribution of the resulting projection images is unchanged under reflection across a fixed line through the origin in the image plane (e.g., the vertical axis). This property can similarly be enforced via data augmentation by including reflected versions of the projection images with respect to that line. In combination with the in-plane uniformity assumption, this chirality invariance with respect to a single axis implies invariance under reflection across any line through the origin. See Section~\ref{sec:basis_for_rotation_distribution} for further structural consequences. 

We aim to relax the uniformity assumption on $\rho_1^*$ in future work to instead only require that both $\rho_1^*$ and $\rho_2^*$ are in-plane uniform. That said, recent experimental advances in cryo-EM (see Section~\ref{sec:introduction}) have introduced techniques that enable nearly uniform orientation distributions in practice.

\vspace{0.2cm}

The paper's specific contributions are summarized as follows:
\begin{itemize}[leftmargin=*]
    \item We introduce a new data fusion framework that leverages two second-order moments computed from datasets with distinct orientation distributions, made possible by recent experimental advances that facilitate collecting multiple datasets of the same molecule under different preparation protocols. To our knowledge, this is the first method to systematically exploit such complementary orientation datasets within the cryo-EM framework. It opens up new opportunities for future developments in multi-dataset reconstruction, orientation-aware inference, and integration with other imaging modalities.

    \item We prove that the molecular structure is \emph{uniquely} identified by the two second-order moments and their corresponding first-order moments (rather than up to a finite ambiguity), provided that the structure is bandlimited. See Section~\ref{sec:precise} for a precise statement. This result significantly improves prior identifiability guarantees that rely on higher-order moments or strong generative assumptions.

    \item We design a practical reconstruction algorithm based on convex relaxation, and demonstrate numerically that it accurately recovers bandlimited structures from the population moments. In our experiments, the algorithm consistently converges to the ground truth without stagnation or spurious local minima, enabling efficient and reliable reconstruction. 
\end{itemize}

\section{Preliminaries}\label{sec:setup}
\subsection{Basis representation of structure}\label{sec:basis_for_structures}

Since the image formation model \eqref{eq:def-rotated-3-D} involves rotations of a 3-D Fourier volume, it is convenient to represent any such Fourier transformed volume $\widehat{\Phi}$ as an element of a function space that is closed under rotations. A natural choice, by the Peter–Weyl theorem (see, e.g., \cite{chirikjian2016harmonic}), is the spherical harmonic basis. In particular, we assume $\widehat{\Phi}$ is square-integrable and supported on a ball of radius $r_{\max}$, and represent it in spherical coordinates $(r, \theta, \varphi)$ using a spherical harmonics expansion for each fixed radius (see, e.g., \cite{sharon2020method}):
\begin{equation}\label{eq:expand_phi_hat_sph_bessel}
\widehat{\Phi}(r,\theta,\varphi) = \sum_{\ell=0}^L\sum_{m=-\ell}^\ell A_{\ell}^m(r) Y_\ell^m(\theta, \varphi), \quad r \in [0, r_\text{max}], \quad \theta \in [0, \pi], \quad \varphi \in [0, 2\pi),
\end{equation}
where $Y_\ell^m$ are the complex-valued spherical harmonics \cite[Eq. 14.30.1]{dlmf}, the positive integer $L$ is a bandlimit parameter, and $A_{\ell }^m(r)$ is a scalar complex-valued function of $r$, serving as the expansion coefficients.

\begin{figure}
    \centering
    \includegraphics[width=\linewidth,
    trim={0cm 6cm 0cm 6cm},
    clip]{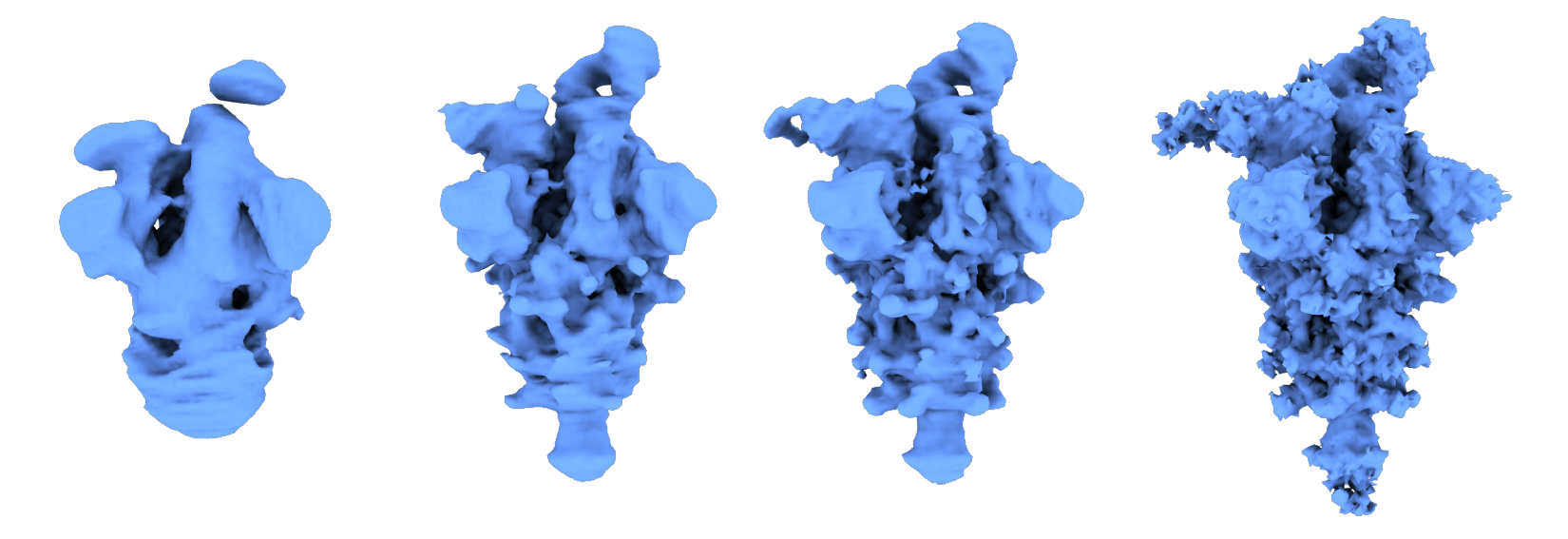}
    \caption{Illustration of the basis expansion of \eqref{eq:expand_phi_hat_sph_bessel}, using $L = 10,20,30$ from leftmost to second-most right figure, and the rightmost figure as the ground-truth.}
    \label{fig:lowpass_structure}
\end{figure}

To motivate the validity of \eqref{eq:expand_phi_hat_sph_bessel}, note that the Dirichlet Laplacian on the ball of radius $r_{\text{max}}$ has eigenfunctions $\psi_{k\ell m}(r,\theta,\phi) := j_\ell(\lambda_{\ell k}\tfrac{r}{r_\text{max}}) Y_\ell^m(\theta, \varphi)$, where $j_\ell$ is the $\ell$th spherical Bessel function of the first kind and $\lambda_{\ell k}$ is its $k$th positive root \cite[\S3.3]{grebenkov2013geometrical}, with $k \in \mathbb{Z}_{>0}, \ell \in \mathbb{Z}_{\geq 0}$ and $m \in \{-\ell , \ldots , \ell\}$. The $\psi_{k\ell m}$'s therefore form a complete orthonormal basis for the space of square-integrable functions on the ball. By summing the $\psi_{k\ell m}$'s over the index $k$, it follows that any square-integrable function supported on the ball of radius $r_\text{max}$ can be represented in the form in \eqref{eq:expand_phi_hat_sph_bessel} as we let $L\to \infty$.

Here we assume that $\widehat{\Phi}$ is bandlimited, in the sense that $\widehat{\Phi}$ can be represented exactly in the form in \eqref{eq:expand_phi_hat_sph_bessel} for a finite value of $L$. Although this is an idealization, should this not hold, higher components $A_{\ell }^m(r) Y_\ell^m(\theta, \Phi)$ for $\ell > L$ can be treated as an additional source of approximation error. An illustration of this bandlimiting assumption is provided in Figure~\ref{fig:lowpass_structure}.

Section~\ref{sec:moment_expressions} presents explicit analytical expressions for the moment $m_2[\Phi, \rho]$ for bandlimited functions of the form \eqref{eq:expand_phi_hat_sph_bessel}, using the complex-valued spherical harmonics introduced above. When describing our algorithm, it will however be convenient to also use real-valued spherical harmonics.  We therefore introduce notation for converting complex-valued spherical harmonics basis coefficients into coefficients in the real-valued spherical harmonics basis.

For each $\ell \in \mathbb{Z}_{\geq 0}$, write  $Q_\ell \in \mathbb{C}^{(2\ell+1)\times(2\ell+1)}$ as a unitary matrix that has non-zero entries only on the main diagonal and anti-diagonal, with non-zero components defined explicitly by
\begin{equation} \label{eq:Q-def}
    (Q_\ell)_{mm} = \begin{cases}
        \ii/\sqrt{2},&  \text{ if } m < 0, \\
        1 & \text{ if } m = 0, \\
        (-1)^m/\sqrt{2},&  \text{ if } m > 0, \\
    \end{cases}  
    \qquad
        (Q_\ell)_{-m,m} = \begin{cases}
        1/\sqrt{2},&  \text{ if } m < 0, \\
        1  & \text{ if } m = 0, \\
        -(-1)^m \ii/\sqrt{2},&  \text{ if } m > 0, \\
    \end{cases}  
\end{equation}
for all $-\ell \leq m\leq \ell$. We denote by $Y_{\ell m}(\theta,\varphi)$ the real-valued spherical harmonics defined by the convention 
\begin{equation}
Y_{\ell m}(\theta,\varphi) = \sum_{m'=-\ell}^{\ell} (Q_{\ell})_{m,m'}Y_\ell^{m'}(\theta,\varphi),
\end{equation}
where the real-valuedness of $Y_{\ell m}$ follows from \cite[Eq.~14.30.1, Eq.~14.9.3]{dlmf}. By changing the angular basis to $Y_{\ell m}(\theta, \varphi)$, we can equivalently write \eqref{eq:expand_phi_hat_sph_bessel} as
\begin{align}
\widehat\Phi(r,\theta,\varphi)=\sum_{\ell=0}^L \sum_{m=-\ell}^\ell A_{\ell m}(r) Y_{\ell m}(\theta,\varphi),
\end{align}
where $A_{\ell m}(r)$ is the corresponding expansion coefficient in the real-valued spherical harmonics basis, obtained by linearly transforming the complex coefficients using the unitary matrix $Q_{\ell}$:
\begin{equation*}
A_{\ell m}(r)=\sum_{m'=-\ell}^{\ell} A_{\ell}^{m'}(r) (\overline{Q}_\ell)_{mm'}.
\end{equation*}
It is equivalent and more convenient to write in the opposite direction that
\begin{equation}\label{eq:def_A_A_tilde}
A_{\ell}^m(r) = \sum_{m'=-\ell}^{\ell} A_{\ell m'}(r) (Q_\ell)_{m'm},
\end{equation}
which can explicitly be written out as
\begin{align*}
A_{\ell}^m(r)= \begin{cases}
        \frac{\ii}{\sqrt{2}}A_{\ell m}(r)+\frac{1}{\sqrt{2}}A_{\ell,-m}(r),&  \text{ if } m < 0, \\
        A_{\ell 0},& \text{ if } m = 0, \\
        \frac{-(-1)^m \ii}{\sqrt{2}}A_{\ell,-m}(r)+\frac{(-1)^m}{\sqrt{2}}A_{\ell m}(r),&  \text{ if } m > 0. \\
    \end{cases}  
\end{align*}
\begin{remark}\label{rmk:real_coeff}
For real-valued structures $\Phi$, symmetry properties of the Fourier transform imply that the coefficients $\{A_{\ell m}(r):-\ell\leq m\leq \ell\}$ are real for even $\ell$ and purely imaginary for odd $\ell$ \cite{bhamre2015orthogonal}.
\end{remark}

\subsection{Basis representation of rotation distribution}\label{sec:basis_for_rotation_distribution}

We model the distribution of viewing orientations as having a density function $\rho(R)$ with respect to the uniform/Haar measure over the rotation group $\mathsf{SO}(3)$. To facilitate analysis and to reflect symmetries commonly present (or enforceable via data augmentation) in cryo-EM experiments, we work under the following two structural assumptions on $\rho(R)$:
\begin{enumerate}
    \item \textbf{In-plane uniformity:} $\rho$ is invariant under right multiplication by any rotation about the $z$-axis, i.e.,
    \begin{align}\label{eq:def_in_plane_uniform}
    \rho(R) = \rho(R z(\alpha)),
    \end{align}
    for all $R\in \mathsf{SO}(3)$ and rotations $z(\alpha)$ of $\alpha\in\R$ radians around the $z$-axis. This implies that the distribution of projection images is invariant to image-plane 2-D rotations.
    
    \item \textbf{Chirality invariance:} $\rho$ is invariant under conjugation by the reflection matrix $J = \mathrm{diag}(1,1,-1)$, i.e.,
    \begin{align}\label{eq:def_invariant_chirality}
    \rho(R) = \rho(J R J),
    \end{align}
    for all $R \in \mathsf{SO}(3)$. This implies that the projection image distribution is invariant under reflection across a fixed axis (e.g., the vertical axis) in the image plane; combined with in-plane uniformity, this further implies invariance under reflection across \emph{any} line through the origin.
\end{enumerate}

We show in Appendix~\ref{app:rotational_distribution} that any square-integrable density $\rho(R)$ on $\mathsf{SO}(3)$ inducing an in-plane uniform distribution of viewing orientations (of the tomographic projection images) can be effectively represented using a bandlimited expansion of order $P$:
\begin{align}\label{eq:def_rho_in_plane_uniform}
\rho(R) \mathrm{d}R =\sum_{p=0}^{P} \sum_{u=-p}^p B_{p,u} U_{u0}^p(R)\mathrm{d}R,
\end{align}
where $\mathrm{d}R$ is the uniform/Haar measure on $\mathsf{SO}(3)$, and $U_{u0}^p(R)$ are Wigner $U$-matrix entries given explicitly by \cite[Eq. 9.42]{chirikjian2016harmonic}
\begin{align}\label{eq:explicit_U_expression}
U_{u0}^p(R)=(-1)^u\sqrt{\frac{4\pi}{2p+1}} \overline{Y_p^u \left(\theta(R), \varphi(R)\right)},
\end{align}
with $(\theta(R), \varphi(R))$ the spherical angles of the third column of $R$. The cutoff $P$ is the bandlimit of the distribution $\rho$; its relationship with the volume bandlimit $L$ will be discussed further in Section~\ref{sec:precise} and Section~\ref{sec:uniqueness}. Figure~\ref{fig:lowpass_dist} illustrates this bandlimited representation of the orientation distribution.

The complex-valued basis coefficients $B_{p,u}$ satisfy constraints from the real-valuedness and normalization of $\rho$:
\begin{align}\label{eq:constraints_on_B}
    \overline{B_{p,u}}=(-1)^u B_{p,-u}, \quad \text{and} \quad B_{0,0} =1,
\end{align}
and chirality invariance further implies
\begin{align}
B_{p,u} = 0, \qquad \text{for all odd } p.
\end{align}
These conditions yield a finite-dimensional representation of $\rho$ whose coefficients lie in a structured subset of complex-valued vectors.

\begin{figure}
    \centering
    \includegraphics[width=\linewidth]{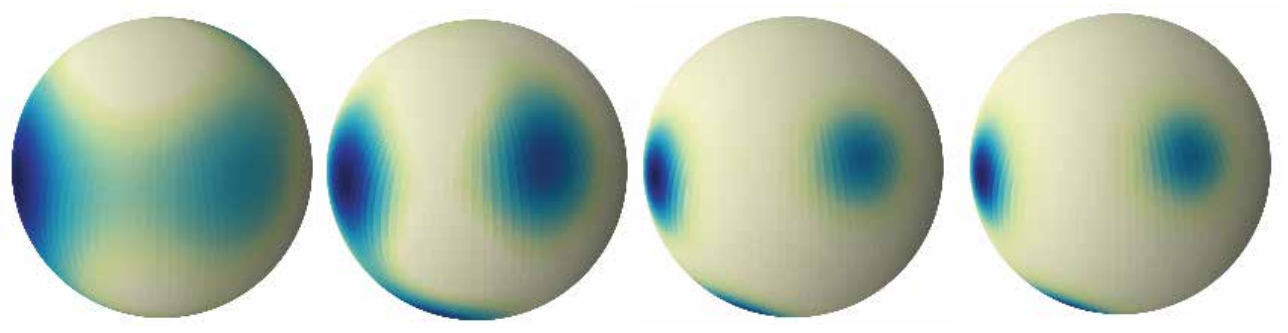}
    \caption{Illustration of the basis expansion of \eqref{eq:def_rho_in_plane_uniform} as function of $\theta(R)$ and $\varphi(R)$, using $P = 3,5,10$ from leftmost to second-most right figure, and the rightmost figure as the ground-truth.}
    \label{fig:lowpass_dist}
\end{figure}

\subsection{Analytic expressions for moments}\label{sec:moment_expressions}
We show in Appendix~\ref{sec:2nd_moment_derivation} that the second-order population moment \eqref{eq:def_pop_moment} for a bandlimited structure of the form \eqref{eq:expand_phi_hat_sph_bessel} with a rotation distribution of the form \eqref{eq:def_rho_in_plane_uniform}, has a succinct form. To present it, denote by $A_{\ell}: [0, r_\text{max}] \rightarrow \C^{1\times (2\ell+1)}$,  for  $\ell \in \{0,\ldots , L\}$, $n \in \{-L, \ldots, L\}$ the row-vector-valued function  defined by 
\begin{align}\label{eq:def_Al}
A_{\ell}(r) = \begin{bmatrix} A_{\ell}^{ -\ell}(r), &  A_{\ell}^{-\ell + 1}(r), & \cdots ,& A_{\ell}^{\ell}(r)\end{bmatrix},
\end{align}
 and define the matrix $\mathcal{B}^n_{\ell,\ell'}\in\C^{(2\ell+1)\times (2\ell'+1)}$, for  $\ell, \ell' \in \{0,\ldots , L\}$, $n \in \{-L, \ldots, L\}$, by
\begin{align}\label{eq:def_calB}
(\mathcal{B}^n_{\ell,\ell'})_{m,m'}&= (-1)^{m+n}\mathcal{N}^n_{\ell} \mathcal{N}^n_{\ell'} \sum_{\ell''=\max\{|m-m'|, |\ell-\ell'|\}}^{\min\{\ell+\ell',P\}} 
\!\!\!\!\!\! \frac{\mathcal{C}_{\ell''} (\ell,\ell',m,-m',n,-n)} {2\ell''+1}B_{\ell'',m'-m} ,
\end{align}
for $-\ell \leq m\leq \ell$ and $-\ell' \leq m'\leq \ell'$. Here, the scalar $\mathcal{N}_\ell^n \in \mathbb{R}$ is defined by
\begin{align}\label{eq:def_calN}
    \mathcal{N}^n_{\ell} &= N^n_{\ell}\cdot \1_{\{\ell\equiv n ~ (\text{mod} \ 2)\}}\cdot \1_{\{\ell\geq |n|\}},
\end{align}
where
\begin{align}\label{eq:formula_N}
N_\ell^n = \sqrt{\frac{2\ell +1}{4\pi}}\sqrt{\frac{(\ell-n)!}{(\ell+n)!}} \cdot P_\ell^n(0),
\end{align}
with $P_\ell^n(x)$ denoting the associated Legendre functions, \cite[Eq.~(14.3.6)]{dlmf}.
The $\mathcal{C}_{\ell''}$ are septuply indexed constants defined (for general indices) by
\begin{equation}\label{eq:myCG0}
\mathcal{C}_{\ell''} (\ell,\ell',m,m',n,n') := C(\ell,m;\ell',m'|\ell'',m+m') C(\ell,n;\ell',n'|\ell'',n+n'),
\end{equation}
where the $C(\cdot, \cdot; \cdot, \cdot  | \cdot, \cdot)$ are Clebsch-Gordan coefficients \cite[Section~9.9]{chirikjian2016harmonic}.

With these definitions in place, we are now ready to state the following proposition giving the expression for the second-order moment. The proof is deferred to Appendix~\ref{sec:2nd_moment_derivation}.
\begin{proposition}\label{prop:2nd_moment} 
The second-order population moment \eqref{eq:def_pop_moment} for a bandlimited structure of the form \eqref{eq:expand_phi_hat_sph_bessel} with a rotation distribution of the form \eqref{eq:def_rho_in_plane_uniform} can be written as
\begin{align}\label{eq:formula_for_m2}
m_2(r,\varphi,r',\varphi') &= \sum_{n=-L}^{L} e^{\ii n(\varphi-\varphi')} \sum_{\ell=0}^L  \sum_{\ell'=0}^L   A_{\ell}(r) \mathcal{B}^n_{\ell,\ell'} (A_{\ell'}(r'))^{\mathsf{H}}.
\end{align}
In particular, $\mathcal{B}^n_{\ell,\ell'}$ satisfies the Hermitian property $\mathcal{B}^n_{\ell,\ell'} = (\mathcal{B}^n_{\ell',\ell} )^{\mathsf{H}}$. Also, $\mathcal{B}^n_{\ell,\ell'}=0$ for any $-L\leq n\leq L$ whenever $\ell\not\equiv \ell' ~(\text{mod}~2)$.
\end{proposition}

\begin{remark}
From the expression of $\mathcal{B}^n_{\ell,\ell'}$ in \eqref{eq:def_calB} and the expansion of $m_2$ in \eqref{eq:formula_for_m2}, we observe that the coefficients $B_{p,u}$ with $p > 2L$ do not appear in $m_2$, due to the constraint $\ell'' \leq \min\{\ell + \ell', P\}$ with $\ell, \ell' \leq L$. Similarly, a direct inspection of the first moment $m_1$ (see Appendix \ref{sec:1st_moment_derivation}) reveals that it depends only on $B_{p,u}$ with $p\leq L$. Therefore, the coefficients $B_{p,u}$ for all $p > 2L$ do not influence the population statistics considered in this work.
\end{remark}

\begin{remark}
In particular, for the uniform distribution of rotations, $B_{\ell,m} = 0$ unless $\ell = 0$ and $m = 0$. Applying \eqref{eq:formula_N}, the parity property of $P_{\ell}^n$, and the explicit formula for Clebsch-Gordan coefficients
\begin{align*}
C(\ell,m;\ell,-m|0,0)=\frac{(-1)^{\ell-m}}{\sqrt{2\ell+1}},~~~C(\ell,n;\ell,-n|0,0)=\frac{(-1)^{\ell-n}}{\sqrt{2\ell+1}},
\end{align*}
then \eqref{eq:formula_for_m2} simplifies to 
\begin{align*}
m_2(r, r', \psi) =\sum_{\ell = 0}^L \sum_{m=-\ell}^\ell A_{\ell}^m(r) \overline{A_{\ell}^m(r')} \sum_{n=-\ell}^\ell \frac{1}{4\pi} \frac{(\ell-n)!}{(\ell+n)!}(P_\ell^n(0))^2 e^{\ii n\psi}.
\end{align*}
Further using the addition theorem for spherical harmonics \cite[Eqn. 4.37]{chirikjian2016harmonic}, we reach
\begin{align}\label{eq:formula_for_uniform_m2}
m_2(r, r', \psi) =\frac{1}{4\pi}\cdot \sum_{\ell = 0}^L \sum_{m=-\ell}^\ell A_{\ell}^m(r) \overline{A_{\ell}^m(r')} P_\ell(\cos\psi),
\end{align}
where the $P_\ell$'s are the Legendre polynomials \cite[\S18.3]{dlmf}, and $\psi := \varphi - \varphi'$. 
\end{remark}

\subsection{Formal problem statement} \label{sec:precise}

We formally formulate the problem of recovering the ground-truth structure $\Phi^*$ from its associated moment observations. We make the following assumptions.

\begin{assumption}(Assumptions on the structure $\Phi$ and the distributions $\rho_1$, $\rho_2$)\label{assumption:distributions}
\begin{enumerate}
\item $\Phi$ is bandlimited with respect to its angular variables, i.e., $\Phi$ can be expressed as in \eqref{eq:expand_phi_hat_sph_bessel} for some $L \geq 3$.
\item The radial functions $\left(A_{\ell}^m(r) : 0 \leq \ell \leq L, - \ell \leq m\leq \ell \right)$ in \eqref{eq:expand_phi_hat_sph_bessel} are linearly independent.
\item The distribution $\rho_1$ for the first dataset is the uniform/Haar measure on $\SO(3)$.
\item The distribution $\rho_2$ for the second dataset is in-plane uniform and invariant to chirality (see Section~\ref{sec:basis_for_rotation_distribution} for precise definitions). Moreover, $\rho_2$ is bandlimited with cutoff $P \geq 2L$, and its expansion coefficients from \eqref{eq:def_rho_in_plane_uniform}
\[
\left(B_{p,u} : 1 \leq p \leq 2L,\ p \text{ even},\ -p \leq u \leq p \right) \in \mathbb{C}^{2L^2 + 3L}
\]
are Zariski-generic.
\end{enumerate}
\end{assumption}

The 3-D reconstruction problem can be formally stated as follows.
\begin{problem}\label{prob:main}
Let $\Phi^*, \rho_1^*, \rho_2^*$ satisfy Assumption~\ref{assumption:distributions}. Given samples of $\widetilde{m}_1[\Phi^*, \rho_1^*]$, $\widetilde{m}_2[\Phi^*, \rho_1^*]$, and $\widetilde{m}_2[\Phi^*, \rho_2^*]$ on a grid in polar coordinates defined by
\begin{equation}\label{eq:def_grid}
\{r_1,\ldots,r_{M_r}\} \subseteq [0, r_\text{max}], \qquad \{\varphi_1, \ldots , \varphi_{M_\varphi}\} \subseteq [0,2\pi),
\end{equation}
recover the structure $\Phi^*$.
\end{problem}

\begin{remark}
    In cryo-EM, datasets with non-uniform viewing directions are readily obtained.  Experimental or data processing techniques are generally required to obtain uniform distributions. In techniques like XFEL, distributions close to uniform are more easily available \cite{saldin2011reconstructing,kurta2013solution,donatelli2015iterative,kurta2017correlations}.
\end{remark}

\section{Algorithm: Method of Double Moments}\label{sec:algorithm}

We present the main computational result: an algorithm for solving Problem~\ref{prob:main} given access only to sample moments evaluated on a discrete polar grid. 
The algorithm consists of three steps:
\begin{enumerate}
\item \textbf{Kam's method (cf. \cite{kam1980reconstruction}):} Use the uniform sample moment $\widetilde{m}_2[\Phi^*, \rho_1^*]$ to recover the structure coefficients $A_\ell(r_i)$ for $i = 1, \ldots , M_r$, up to the action of unknown orthogonal matrices $O_\ell \in \mathsf{O}(2\ell + 1)$ for $\ell=0, \ldots , L$, i.e., the problem reduces to estimating the set $\{O_\ell\}$.

\item \textbf{Formulate optimization:} Use the non-uniform sample moment $\widetilde{m}_2[\Phi^*, \rho_2^*]$ to set up a non-linear least-squares problem in the variables $\{O_\ell\}$ and $\{B_{p,u}\}$.

\item \textbf{Convex relaxation and alternating refinement:} Solve the optimization problem by iteratively updating $\{O_\ell\}$ and $\{B_{p,u}\}$, 
while regularizing the constraint using information from the first-order moment 
$\widetilde{m}_1[\Phi^*, \rho_1^*]$.
\end{enumerate}
These steps are detailed in Sections~\ref{sec:method_step_1} -- \ref{sec:method_step_3}, respectively. Algorithm~\ref{alg:modm} presents the complete procedure in summary form. Here, all quantities with a tilde $\,\widetilde{\cdot}\,$ are empirically accessible from the observed samples.

\begin{algorithm}
\caption{Method of double moments (MoDM)}\label{alg:modm}
\DontPrintSemicolon
\KwIn{Sample moments $\widetilde{m}_1[\Phi^*, \rho_1^*]$, $\widetilde{m}_2[\Phi^*, \rho_1^*]$, and $\widetilde{m}_2[\Phi^*, \rho_2^*]$ satisfying Assumption~\ref{assumption:distributions} and sampled on the grid \eqref{eq:def_grid}, maximum number of iterations $K$.}
\KwResult{Approximation $\Phi^{(K)}$ to the underlying molecular structure $\Phi^*$.}

Compute $\widetilde{C}$ in \eqref{eq:def_C_tilde} through numerical quadrature, and obtain $\widetilde{A}_\ell$ from $\widetilde{C}_\ell$ via Cholesky decomposition. \Comment*[r]{Step $1$}

Compute $\widetilde{G}^n$ in \eqref{eq:def_Gn_and_Mn} through numerical quadrature, and obtain $\widetilde{\mathcal{M}}^n$ in \eqref{eq:2nd_moment_info} by left- and right-multiplying $\widetilde{G}^n$ by $\widetilde{A}^\dagger$ and $\widetilde{A}^{\mathsf{H},\dagger}$, respectively. \Comment*[r]{Step $2$}

Compute iterative solutions $O_\ell^{(k)}$ and $B_{p,u}^{(k)}$ to the update equations \eqref{eq:updateB} -- \eqref{eq:updateO}, for $k = 1, \ldots , K$. \Comment*[r]{Step $3$}

Return the inverse discrete Fourier transform of
$$
\widehat{\Phi}^{(K)}(r_i,\theta,\varphi) := \sum_{\ell = 0}^L\sum_{m=-\ell}^{\ell} (\widetilde{A}_\ell O_\ell^{(K)})_{i,m} Y_{\ell m}(\theta, \varphi),
$$
sampled at the radial points $r_i$ from \eqref{eq:def_grid} for $i = 1, \ldots , M_r$.
\end{algorithm}

\subsection{Step 1: Kam’s method via uniform sample moment}\label{sec:method_step_1}

Kam \cite{kam1980reconstruction} showed, via the Fourier slice theorem, that if the viewing directions are uniformly distributed over $\SO(3)$, then the autocorrelation function of the 3-D volume with itself can be computed directly from the second-order moment of 2-D projection images. In particular, for a uniform distribution $\rho_1^*$, it follows from \eqref{eq:formula_for_uniform_m2} and the orthogonality of the Legendre polynomials that, for each $\ell = 0,\ldots,L$ and $r,r'\in[0,r_\text{max}]$,
\begin{align}\label{eq:def_Kam}
C_\ell(r,r') \!&:=\! 2\pi(2\ell+1)\!\cdot\!\int_0^{\pi} m_2[\Phi, \rho_1^*](r, r', \psi) P_\ell(\cos \psi) \sin (\psi )d\psi \notag\\
&= \sum_{m=-\ell}^{\ell}A_{\ell}^m(r)\overline{A_{\ell}^m(r')}=\sum_{m=-\ell}^{\ell}A_{\ell m}(r) A_{\ell m}(r'),
\end{align}
where the last equation comes from the unitarity of $Q_\ell$ in Section \ref{sec:basis_for_structures}.  This relation provides partial information about the unknown coefficients $A_{\ell}^m$ or $A_{\ell m}$.

Assume we have access to samples on the grid in \eqref{eq:def_grid}. 
We approximate the integral in \eqref{eq:def_Kam} by firstly using the sample moment $\widetilde{m}_2[\Phi^*,\rho_1^*]$ instead of the population moment, and secondly by using numerical quadrature with quadrature points $\varphi_s, \varphi_q$ and quadrature weights $w_{s,q}$, for $s,q = 0, \ldots , M_\varphi$, i.e., we compute 
\begin{align}\label{eq:def_C_tilde}
    \widetilde{C}_\ell(r_i,r_j) \!:=\! 2\pi(2\ell+1)\!\cdot\!\sum_{s,q=0}^{M_\varphi} w_{s,q} \widetilde{m}_2[\Phi^*, \rho_1^*](r_i, r_j, \varphi_s - \varphi_q) P_\ell(\cos (\varphi_s - \varphi_q ))\sin (\varphi_s - \varphi_q ) .
\end{align}
Next, for each $\ell=0,\ldots,L$, we introduce the matrices 
$A_\ell \in \C^{M_r\times (2\ell+1)}$, 
$\widecheck{A}_\ell \in \C^{M_r\times (2\ell+1)}$, 
$C_\ell\in \R^{M_r\times M_r}$, and $\widetilde{C}_\ell \in \C^{M_r\times M_r}$. They are defined entrywise as
\begin{equation}\label{eq:defAl}
   (A_\ell)_{k,m} = A_\ell^m(r_k), \quad 
   (\widecheck{A}_\ell)_{k,m} = A_{\ell m}(r_k), \quad 
   (C_\ell)_{i,j} = C_\ell(r_i,r_j), \quad 
   (\widetilde{C}_\ell)_{i,j} = \widetilde{C}_\ell(r_i,r_j),
\end{equation}
respectively, where 
$\{r_1,\ldots,r_{M_r}\}$ denote the radial sampling points. 
Inserting these definitions into \eqref{eq:def_Kam} and \eqref{eq:def_C_tilde} gives
\begin{equation}\label{eq:Cl_tilde_and_AlAlH}
    \widetilde{C}_\ell \approx C_\ell = A_\ell A_\ell^{\mathsf{H}} = \widecheck{A}_\ell \widecheck{A}_\ell^{\mathsf{H}},
\end{equation}
where the approximation reflects discretization and sampling errors.

Recall from Remark~\ref{rmk:real_coeff} that $\widecheck{A}_\ell$ is real for even~$\ell$ and purely imaginary for odd~$\ell$. 
Applying the Cholesky decomposition of the matrix $C_\ell$, we then obtain factors
\[
\breve{A}_\ell \in \C^{M_r\times (2\ell+1)},
\]
which can be chosen real for even~$\ell$ and purely imaginary for odd~$\ell$. Moreover, there exists a unique real orthogonal matrix $O_\ell \in \mathsf{O}(2\ell+1)$ such that
\begin{equation}\label{eq:Aring_relation}
\breve{A}_\ell = \widecheck{A}_\ell O_\ell^\T = A_\ell Q_\ell^{\mathsf{H}} O_\ell^\T.
\end{equation}

In practice, we replace $C_\ell$ with its sample counterpart $\widetilde{C}_\ell$, and perform a Cholesky decomposition, to obtain an estimate $\widetilde{A}_\ell$ for $\breve{A}_\ell$. The remaining task is to recover the orthogonal matrices $\{O_\ell\}$, which would in turn allow us to fully estimate $\{A_\ell\}$.

\subsection{Step 2: Formulate optimization with the non-uniform sample moment}\label{sec:method_step_2}
The remainder of the computational method attempts to recover the matrices $O_\ell$, for $0\leq \ell \leq L$ by matching the additional non-uniform sample moment $\widetilde{m}_2[\Phi^*,\rho_2^*]$ to the corresponding population moment. We therefore re-express $m_2[\Phi, \rho_2]$ in terms of the matrices $\breve{A}_\ell$ computed in Section~\ref{sec:method_step_1}. Inserting \eqref{eq:Aring_relation} into \eqref{eq:formula_for_m2} gives
\begin{align}
m_2(r,\varphi,r',\varphi')
&= \sum_{n=-L}^{L} e^{\ii n(\varphi-\varphi')} \sum_{\ell = 0}^{L }  \sum_{\ell' = 0}^ {L }   \breve{A}_{\ell}(r) O_\ell Q_\ell \mathcal{B}^n_{\ell,\ell'} Q_{\ell'}^{\mathsf{H}}O_{\ell'}^\T (\breve{A}_{\ell'}(r'))^{\mathsf{H}},\label{eq:2nd_moment_matrix_O}
\end{align}
where we denote by $\breve{A}_\ell(r)$ the row of $\breve{A}_\ell$ corresponding to the radius $r$.

Note that performing the integration over $\varphi$ and $\varphi'$ yields
\begin{align*}
    &\frac{1}{(2\pi)^2}\int_0^{2\pi}\!\!\!\int_0^{2\pi} 
    m_2[\Phi, \rho_2](r,\varphi,r',\varphi') \,
    e^{-i n(\varphi-\varphi')} \, \mathrm{d}\varphi \, \mathrm{d}\varphi' = \sum_{\ell = 0}^{L }  \sum_{\ell' = 0}^{L }   
    \breve{A}_{\ell}(r) O_\ell Q_\ell
    \mathcal{B}^n_{\ell,\ell'}Q_{\ell'}^{\mathsf{H}} O_{\ell'}^\T 
    (\breve{A}_{\ell'}(r'))^{\mathsf{H}}.
\end{align*}
Thus, the integral can be approximated numerically by first replacing 
$m_2$ with the sample moment $\widetilde{m}_2[\Phi^*,\rho_2^*]$ and then 
applying numerical quadrature, which yields
\begin{equation}
\begin{split}\label{eq:gn}
   \widetilde{G}^n(r_i,r_j) &:= \frac{1}{(2\pi)^2}\sum_{s,q=0}^{M_\varphi} w_{s,q}  
   \widetilde{m}_2[\Phi^*, \rho_2^*](r_i, r_j, \varphi_s - \varphi_q) 
   e^{-\ii n(\varphi_s - \varphi_q)} \\
   &\approx \sum_{\ell, \ell' = 0}^L 
   \breve{A}_{\ell}(r_i) O_\ell Q_\ell \mathcal{B}^n_{\ell,\ell'} 
   Q_{\ell'}^{\mathsf{H}} O_{\ell'}^\T (\breve{A}_{\ell'}(r_j))^{\mathsf{H}},
\end{split}
\end{equation}
for each $-L \le n \le L$, where the approximate equality arises from 
discretization and sampling errors.

Define the matrices $\widetilde{G}^n \in \C^{M_r \times M_r}$ and 
$\mathcal{M}^n_{\ell, \ell'} \in \C^{(2\ell+1)\times(2\ell'+1)}$ by
\begin{align}\label{eq:def_Gn_and_Mn}
(\widetilde{G}^n)_{i,j} = \widetilde{G}^n(r_i,r_j), \qquad
\mathcal{M}^n_{\ell,\ell'} = O_\ell Q_\ell \mathcal{B}^n_{\ell,\ell'} Q_{\ell'}^{\mathsf{H}} O_{\ell'}^\T.
\end{align}
Let $\mathcal{M}^n$ and $\mathcal{B}^n \in \C^{(L+1)^2 \times (L+1)^2}$ 
be the block matrices whose $(\ell,\ell')$-th blocks are $\mathcal{M}^n_{\ell,\ell'}$ 
and $\mathcal{B}^n_{\ell,\ell'}$, respectively, i.e.,
\begin{equation}\label{eq:def_Mn_and_Bn_blocks}
\mathcal{M}^n = 
\begin{bmatrix}
\mathcal{M}^n_{0,0} & \cdots & \mathcal{M}^n_{0,L} \\
\vdots & \ddots & \vdots \\
\mathcal{M}^n_{L,0} & \cdots & \mathcal{M}^n_{L,L}
\end{bmatrix}, \qquad
\mathcal{B}^n =
\begin{bmatrix}
\mathcal{B}^n_{0,0} & \cdots & \mathcal{B}^n_{0,L} \\
\vdots & \ddots & \vdots \\
\mathcal{B}^n_{L,0} & \cdots & \mathcal{B}^n_{L,L}
\end{bmatrix}.
\end{equation}
Further define the block matrices 
\[
\breve{\mathcal{A}} = \begin{bmatrix} \breve{A}_0 & \cdots & \breve{A}_L \end{bmatrix},
\qquad 
\widetilde{\mathcal{A}} = \begin{bmatrix} \widetilde{A}_0 & \cdots & \widetilde{A}_L \end{bmatrix}
\in \C^{M_r \times (L+1)^2},
\]
where $\breve{A}_\ell$ are the ideal factors and $\widetilde{A}_\ell$ their estimates obtained from 
$\widetilde{C}_\ell$. Then the approximate relation in~\eqref{eq:gn} can be written as
\begin{align}
\widetilde{G}^n \approx \breve{\mathcal{A}} \mathcal{M}^n \breve{\mathcal{A}}^{\mathsf{H}}
\approx \widetilde{\mathcal{A}} \mathcal{M}^n \widetilde{\mathcal{A}}^{\mathsf{H}}.
\end{align}
In the expression above, the sample moment information provides us with 
$\widetilde{G}^n$ and $\widetilde{\mathcal{A}}$, whereas our goal is to recover 
the unknown matrices $\{O_\ell\}$ and $\{\mathcal{B}^n_{\ell,\ell'}\}$ contained 
within $\mathcal{M}^n$.

Assuming $\widetilde{\mathcal{A}}$ has full rank (which in particular requires that $M_r \geq (L+1)^2)$, we can left- and right-multiply this equation by the pseudoinverse matrices $\widetilde{\mathcal{A}}^\dagger$ and $\widetilde{\mathcal{A}}^{\mathsf{H,\dagger}}$ to obtain approximate access to $\mathcal{M}^n$. Writing
$$\mathcal{O} = \text{blockdiag}_{\ell = 0 , \ldots , L}(O_\ell), \quad \text{ and } \quad \mathcal{Q} = \text{blockdiag}_{\ell = 0 , \ldots , L}(Q_\ell),$$
we obtain
\begin{align}
\widetilde{\mathcal{A}}^\dagger \widetilde{G}^n \widetilde{\mathcal{A}}^{\mathsf{H},\dagger}:=\widetilde{\mathcal{M}}^n \approx \mathcal{M}^n = \mathcal{O} \mathcal{Q} \mathcal{B}^n \mathcal{Q}^{\mathsf{H}} \mathcal{O}^\T, ~~~{\rm for}~~~ -L \leq n \leq L.\label{eq:2nd_moment_info}
\end{align}
This naturally leads to a least-squares problem, where we seek to find 
the matrices $\{O_\ell\}$ and parameters $\{B_{p,u}\}$ by minimizing
\begin{align}\label{eq:ls_problem}
\sum_{n=-L}^{L} 
\Big\| \widetilde{\mathcal{M}}^n
- \mathcal{O} \mathcal{Q} \mathcal{B}^n \mathcal{Q}^{\mathsf{H}} \mathcal{O}^\T \Big\|_F^2.
\end{align}
In practice, it is necessary for the  distribution $\rho$ to be sufficiently far from uniform for the least-squares problem \eqref{eq:ls_problem} to be reasonably conditioned.

\subsection{Step 3: Convex relaxation and alternating refinement}\label{sec:method_step_3}
We solve the problem~\eqref{eq:ls_problem} through an alternating procedure for the matrices $O_\ell$ and the parameters $B_{p, u}$, by initializing $O_\ell^{(0)}$ as arbitrary orthogonal matrices and updating the parameters $\{B_{p,u}^{(k)}\}$ and $\{O_\ell^{(k)}\}$ iteratively, for $k = 0, \ldots , K$.

\paragraph{Solving for $\{B_{p,u}\}$.} 
Given $\mathcal{O}^{(k)}=\text{blockdiag}_{\ell = 0 , \ldots , L}(O_\ell^{(k)})$, the update of $\{B_{p,u}\}$ proceeds by minimizing the expression
\begin{equation}\label{eq:minimizing_B_first_one}
     \sum_{n=-L}^L \Big\| \widetilde{\mathcal{M}}^n - \mathcal{O}^{(k)}\mathcal{Q} \mathcal{B}^n \mathcal{Q}^{\mathsf{H}} \mathcal{O}^{(k),\T} \Big\|_F^2,
\end{equation}
in the variables $\{B_{p,u}\}$ while imposing the constraints \eqref{eq:constraints_on_B}, i.e., by restricting $B$ to the set
\begin{equation} \mathcal{S_B} := \label{eq:constraint_set_for_B}
    \Big\{ \{B_{p,u}\}_{(p,u) \in \mathcal{I}} : B_{p,u} \in \C,  \quad \overline{B_{p,u}}=(-1)^u B_{p,-u}, \quad  B_{0,0} =1\Big\},
\end{equation}
where $\mathcal{I} := \{(p,u) \subseteq \mathbb{Z}_{\geq 0} \times \mathbb{Z} : -p \leq u \leq p,\, p\leq P\}$. Note that the constraint set $\mathcal{S_B}$ does not impose positivity of the resulting density in \eqref{eq:def_rho_in_plane_uniform}, for simplicity. This can however be incorporated by, for instance, requiring positivity at a given set of collocation points $R_i$ (see \cite[Eq.~(51)]{sharon2020method}). This yields the constraints $\rho(R_i) \geq 0$, which are linear constraints in the variables $\{B_{p,u}\}$ and therefore can be included when minimizing \eqref{eq:minimizing_B_first_one}, although at the expense of increased runtime.

\paragraph{Solving for $\{O_\ell\}$.}

Given the coefficients $B_{p,u}^{(k)}$, we obtain $\mathcal{B}_{\ell,\ell'}^{n,(k)}$ from \eqref{eq:def_calB} and form the block matrix $\mathcal{B}^{n,(k)}$ as in \eqref{eq:def_Mn_and_Bn_blocks}. The update of $\{O_\ell\}$ then follows a relaxed procedure. 

First, note that
\begin{equation}
    \big\| \widetilde{\mathcal{M}}^n - \mathcal{O}\,\mathcal{Q} \,\mathcal{B}^{n,(k)}\,\mathcal{Q}^{\mathsf{H}} \,\mathcal{O}^{\T}  \big\|_F 
    = \big\| \widetilde{\mathcal{M}}^n \mathcal{O} - \mathcal{O}\,\mathcal{Q} \,\mathcal{B}^{n,(k)}\,\mathcal{Q}^{\mathsf{H}} \big\|_F.
\end{equation}
By relaxing the orthogonality constraint on $\mathcal{O}$, we introduce a relaxed variable 
\begin{align*}
\mathcal{X}=\text{blockdiag}_{\ell=0, \ldots , L}(X_\ell),
\end{align*}
where each $X_\ell$ has the same matrix size as $O_\ell$. We then solve for $\mathcal{X}$ in a least-squares sense and subsequently orthogonalize the result to recover $\mathcal{O}$. 
Specifically, we restrict the search to the following set:
\begin{equation}
\begin{split}
    \mathcal{S_X} = \Big\{ \{X_\ell\}_{\ell=0,\ldots, L} : \; 
    X_\ell \in \mathbb{R}^{(2\ell+1) \times (2\ell+1)}, \; X_0 = 1, \; X_1 = I_3 \Big\}.
\end{split}
\end{equation}
This restriction is without loss of generality, because for any choice of $X_0$ and $X_1$, there exists a rotation of the underlying structure $\Phi^*$ that maps $X_0$ to $1$ and $X_1$ to $I_3$ simultaneously (see \cite[pg.~324]{chirikjian2016harmonic}). 

Given a solution $\mathcal{X}$ in $\mathcal{S_X}$, we orthogonalize it by projecting onto
\begin{equation}
\begin{split}
    \mathcal{S_O} = \Big\{ \{O_\ell\}_{\ell=0, \ldots , L}  : \; O_\ell \in \mathsf{O}(2\ell+1) \Big\},
\end{split}
\end{equation}
which is equivalent to solving the constrained least-squares problem
\begin{equation}
     \min_{\{O_\ell\} \in \mathcal{S_O}}\|\mathcal{O}- \mathcal{X} \|_F^2 
    = \min_{\{O_\ell\} \in \mathcal{S_O}} \sum_{\ell = 0}^L \| O_\ell - X_\ell \|_F^2,
\end{equation}
where the equality follows from the block-diagonal structures of $\mathcal{X}$ and $\mathcal{O}$. Each term in the sum can be minimized efficiently using the orthogonal Procrustes procedure, 
which admits a closed-form solution by computing the singular value decomposition (SVD) of each diagonal block $X_\ell$~\cite{schonemann1966generalized}.

\paragraph{Summary.} Taken together, the update equations for $\{B_{p,u}\}$ and $\{O_\ell\}$ can be written as
\begin{align}
    \{B_{p,u}^{(k+1)}\}_{(p,u)\in\mathcal{I}} &= \argmin_{\{B_{p,u}\} \in \mathcal{S_B}} \sum_{n=-L}^L \| \widetilde{\mathcal{M}}^n - \mathcal{O}^{(k)}\mathcal{Q} \mathcal{B}^n\mathcal{Q}^{\mathsf{H}}\mathcal{O}^{(k),\T} \|_F^2,\label{eq:updateB} \\
    \{X_\ell^{(k+1)}\}_{\ell=0,\ldots,L} &= \argmin_{\{X_\ell\}\in \mathcal{S_X}}  \sum_{n=-L}^L \|\widetilde{\mathcal{M}}^n \mathcal{X} - \mathcal{X} \mathcal{Q} \mathcal{B}^{n,(k)}\mathcal{Q}^{\mathsf{H}}\|_F^2, \label{eq:updateX}\\
    \{O_\ell^{(k+1)}\}_{\ell=0,\ldots,L} &= \argmin_{ \{O_\ell\} \in \mathcal{S_O}} \sum_{\ell=0}^L \| O_\ell - X_\ell^{(k+1)}\|_F^2 .\label{eq:updateO}
    \end{align}
Equations~\eqref{eq:updateB} and \eqref{eq:updateX} are structured least squares-problems in the variables $B_{p,u}$ and $X_\ell$, respectively, and can be solved efficiently. The orthogonalization procedure in equation~\eqref{eq:updateO} is a sequence of $L+1$ orthogonal Procrustes problems.

\subsection{Computational complexity}
Estimation of the second order moments from $N$ images of size $M\times M$ takes $\bigO(NM^3 + M^4)$ operations \cite{marshall2023fast}. Step 1 of Algorithm~\ref{alg:modm} takes $\bigO(LM_r^2M_\varphi^2)$ operations and step 2 has complexity $\bigO(M_r^2L^3)$. Step 3 takes $\bigO(LM_r^2M_\varphi^2)$ operations. Forming the matrices $\mathcal{B}^n$ and $\breve{\mathcal{A}}$ takes time $\bigO(L^2)$ and $\bigO(M_rL^2)$, respectively. Forming the matrices $\mathcal{O}^{(k)}\mathcal{Q} \mathcal{B}^n \mathcal{Q}^{\mathsf{H}} \mathcal{O}^{(k),\T}$ and $\mathcal{O}\mathcal{Q} \mathcal{B}^{n,(k)} \mathcal{Q}^{\mathsf{H}} \mathcal{O}^{\T}$ in each iteration takes time $\bigO(L\sum_{\ell, \ell'=0}^L \ell^2 \ell' + \ell\ell'^2) = \bigO(L^6)$, for a total of $\bigO(KL^6)$ operations when running $K$ iterations. Solving the update equation for the parameters $\{B_{p,u}\}$ takes time $\bigO(L^6 + L^5P^4 + P^3) = \bigO(L^9)$, when using all $\bigO(L^5)$ available equations for the $\bigO(L^2)$ variables. Solving the update equation for $X$ takes time $\bigO(L^6 + L^5L^6 + L^9) = \bigO(L^{11})$, when using all $\bigO(L^5)$ available equations for the $\bigO(L^3)$ variables. Solving the update equation for the $O_\ell$ takes time $\bigO(\sum_{\ell=0}^L \ell^3)  = \bigO(L^4)$.

The above complexities do not take several structural features into account. For instance, it is possible to use fewer than the $\bigO(L^5)$ equations to solve for the $\bigO(L^3)$ variables when updating $X$. Moreover, the linear systems exhibit sparsity and iterative solvers can therefore be used for lower complexity (and this is an option in the accompanying code). Lastly, the linear systems have additional structure and are in fact coupled Sylvester-type equation for which there are specialized linear solvers with improved complexity.

By comparison, the complexity of one iteration in the EM algorithm is $\bigO(NM^4\log M)$ (ignoring translational shifts). 
For $L = \bigO(1)$, the method-of-double-moments therefore costs less than EM by a factor of $\bigO(N \log M)$ per iteration.  Moreover, the upfront cost of moment formation is less than the complexity of one EM iteration by a factor of $\bigO(M \log M)$.

\subsection{Numerical experiments}\label{sec:numerics}
This section demonstrates the Algorithm of Section~\ref{sec:algorithm} using the two simulated datasets EMD-2660 \cite{wong2014cryo} and EMD-32743 \cite{guo2022structures} from the online electron microscopy data bank \cite{lawson2016emdatabank}. 

For varying $M$, we generate $100,000$ projection images of size $M\times M$, with signal-to-noise-ratio $0.1$ and affected by convolution with point spread functions with $20,000$ defocus groups. The projection images were generated with voltage 300 $kV$, defocus values ranging between $1$ $\mu m$ and $3$ $\mu m$, amplitude contrast $0.1$ and spherical aberration $2$ mm. The ground truth distribution of viewing angles was taken to be a mixture of $8$ von Mises-Fisher-Langevin distributions.  These mixtures provide a natural mathematical model for clustered preferred-orientation patterns that have been observed in experimental cryo-EM studies; see, e.g., \cite[Figure~2]{li2021effect}, \cite[Figure~6]{aiyer2024overcoming}, and \cite[Figure~2]{jenkins2025overcoming}.

A representative result of running Algorithm~\ref{alg:modm} on bandlimited structures is shown in Figure~\ref{fig:numerical_results1} and Figure~\ref{fig:numerical_results2}, where the ground truth structures are taken to be bandlimited versions of EMD-2660 and EMD-32743, projected to a maximum of bandlimit $L$. Throughout this section, we refer to the original untruncated volumes simply as original structures EMD-2660 and EMD-32743.  The structures used to generate images are their bandlimited versions.

In the experiments, we observe that the algorithm typically successfully reconstructs bandlimited structures only up to a maximal bandlimit \(L(M)\), which depends on the number of pixels in the projection images. The figures display the largest empirically successful bandlimit \(L(M)\) for image sizes \(M\in\{64,128,256\}\), together with the reconstruction obtained at the next bandlimit \(L(M)+1\), where the method fails. That is, given an image resolution \(M\), when the ground-truth volume is truncated at the next bandlimit level \(L(M)+1\), the reconstruction produced by our method exhibits a clear degradation in quality, indicating unsuccessful recovery beyond the empirically achievable bandlimit. Understanding the precise dependence between the achievable reconstruction resolution and the resolution of the projection images is an important direction for future work. Successful reconstruction also depends on using sufficiently many projection images to accurately estimate the moments, as well as on the viewing-angle distribution being sufficiently non-uniform. In practice, certain viewing distributions can lead to unsuccessful reconstructions even for bandlimits below \(L(M)\). This behavior is related to the inversion of the linear system in Eq.~\eqref{eq:2nd_moment_info}, which is only well-posed for a limited range of bandlimits \(L\) for a fixed image resolution \(M\). In addition, the optimization landscape depends on the viewing-angle distribution coefficients \(\{B_{p,u}\}\). Intuitively, successful reconstruction requires a viewing-angle distribution that is sufficiently distinct from the uniform distribution, while still avoiding extreme concentration or regions with near-zero probability mass.

\begin{figure}[!h]
    \centering
    \includegraphics[width=\linewidth,
    trim={0cm 5cm 0cm 5cm},
    clip]{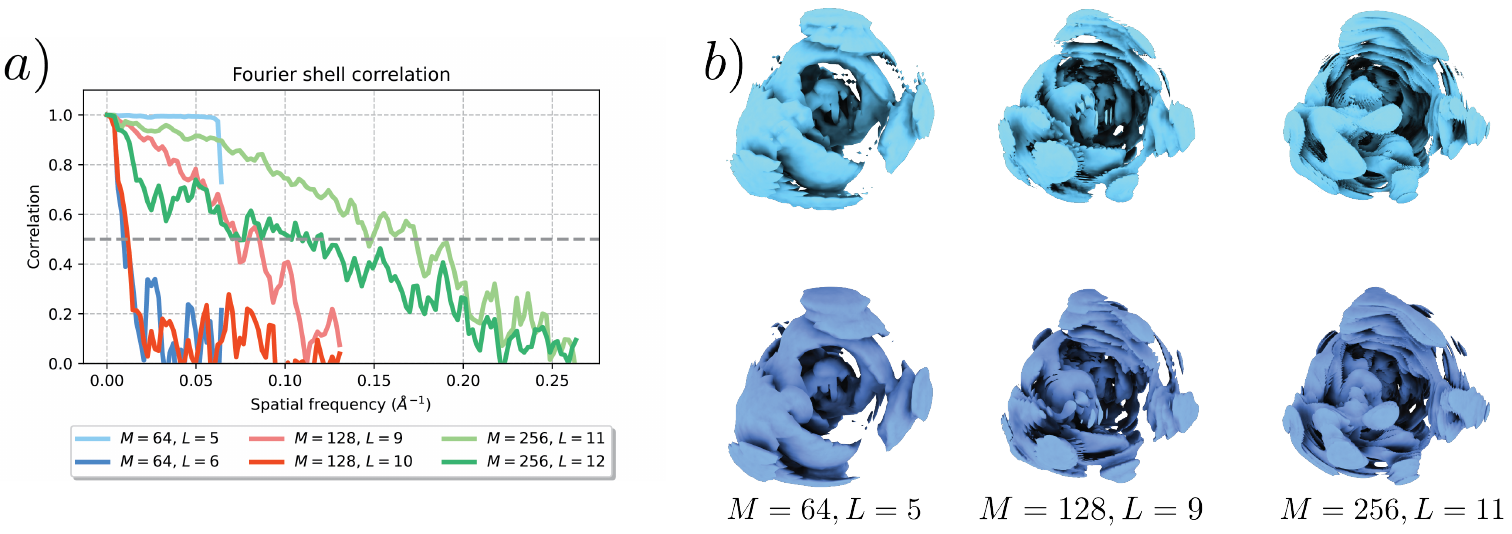}
    \caption{Result of running Algorithm~\ref{alg:modm} using as ground truth structures, bandlimited versions of the dataset EMD-2660 \cite{wong2014cryo} from the online Electron Microscopy Data Bank \cite{lawson2016emdatabank}. (a) FSC curves between the reconstructions and the corresponding ground-truth structures, for different values of bandlimit \(L\) and image size \(M\). (b) Reconstructions (top) and the corresponding ground-truth structures (bottom).} 
\label{fig:numerical_results1}
\end{figure}

\begin{figure}[!h]
    \centering
    \includegraphics[width=\linewidth,
    trim={0cm 5cm 0cm 5cm},
    clip]{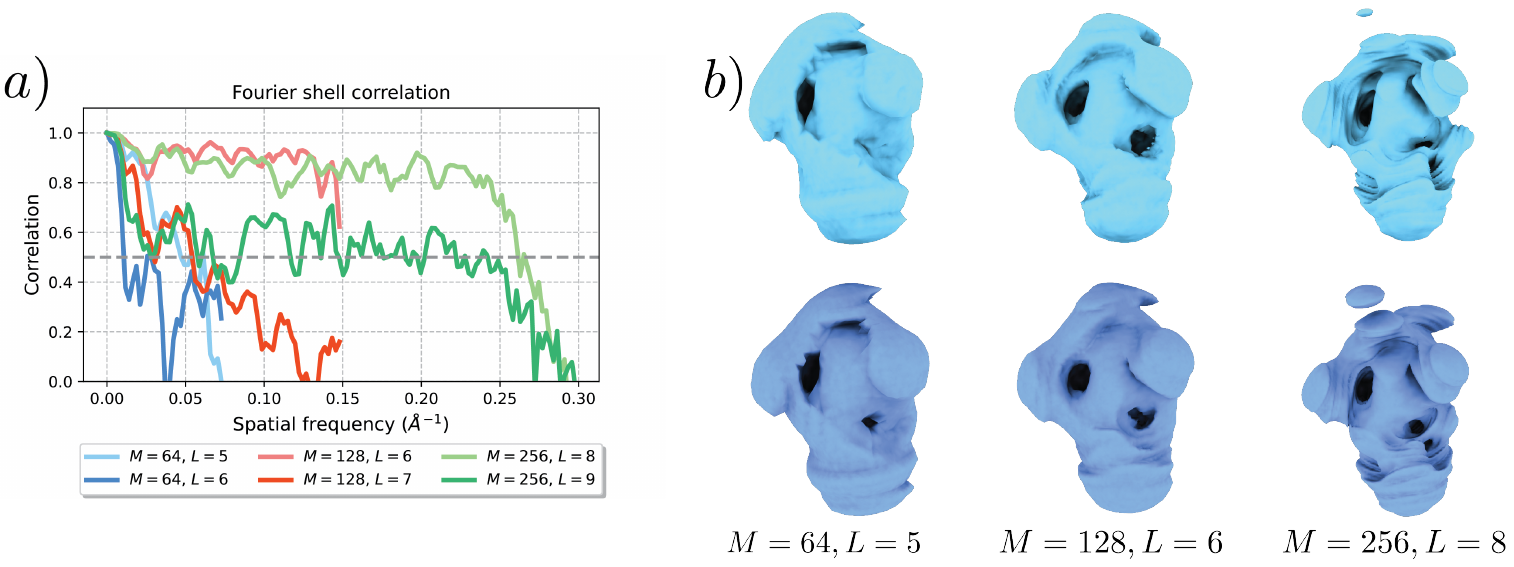}
    \caption{Result of running Algorithm~\ref{alg:modm} using as ground truth structures, bandlimited versions of the dataset EMD-32743 \cite{guo2022structures}  from the online Electron Microscopy Data Bank   \cite{lawson2016emdatabank}. (a) FSC curves between the reconstructions and the corresponding ground-truth structures, for different values of bandlimit \(L\) and image size \(M\). (b) Reconstructions (top) and the corresponding ground-truth structures (bottom).}
\label{fig:numerical_results2}

\end{figure}

As an additional comparison, Figure~\ref{fig:res_vs_not_trunc} reports the FSC between the reconstruction and the original structures EMD-2660 and EMD-32743. Similarly, Figure~\ref{fig:trunc_vs_not_trunc} reports the FSC between the ground-truths used in the experiments and the original structures EMD-2660 and EMD-32743. Each figure includes results for both datasets. The two sets of curves exhibit similar resolution behavior. This suggests that the dominant source of discrepancy from the original structures arises from the spherical-harmonic truncation itself, rather than from reconstruction error. To further study the discrepancy introduced by spherical-harmonic truncation at different bandlimits, we perform an additional numerical experiment in Figure~\ref{fig:truncation_error}.  This test shows that increasing the bandlimit \(L\) leads to smaller truncation error.

\begin{figure}[!ht]
    \centering
\includegraphics[width=1\linewidth]{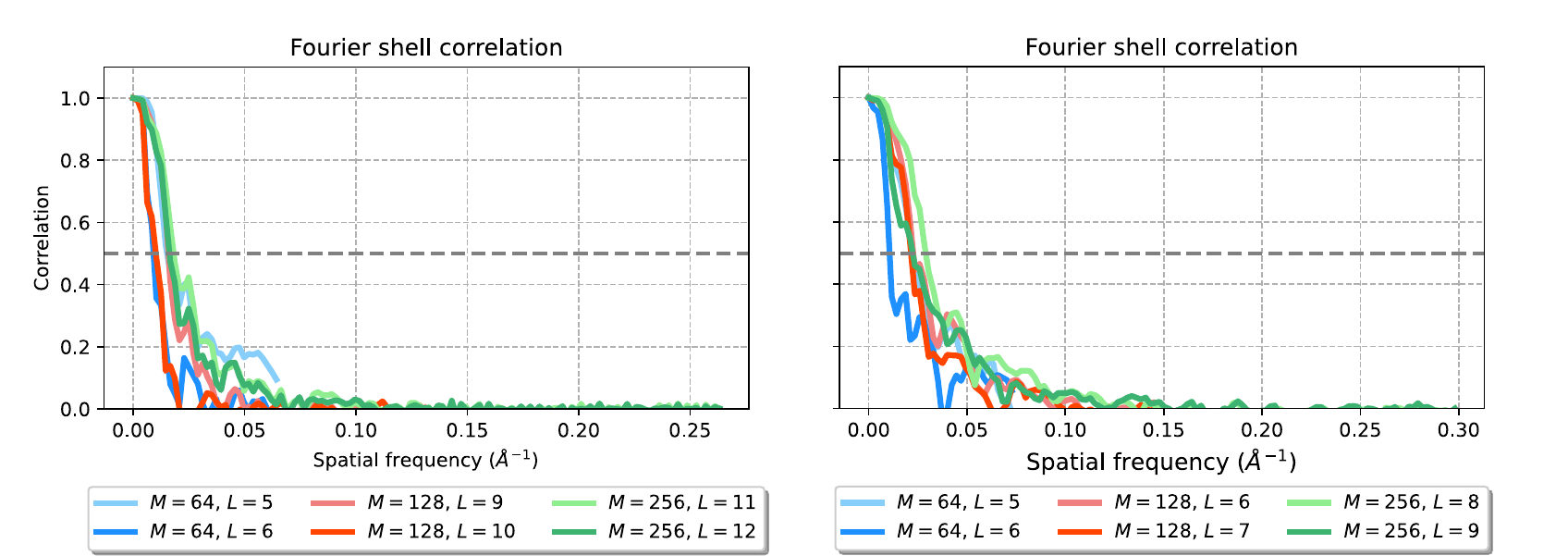}
    \caption{FSC curves between reconstruction results and the original structures EMDB-2660 (left) and EMDB-32743 (right).}
    \label{fig:res_vs_not_trunc}
\end{figure}

\begin{figure}[!ht]
    \centering
\includegraphics[width=1\linewidth, trim=0 2cm 0 1.5cm,
clip]{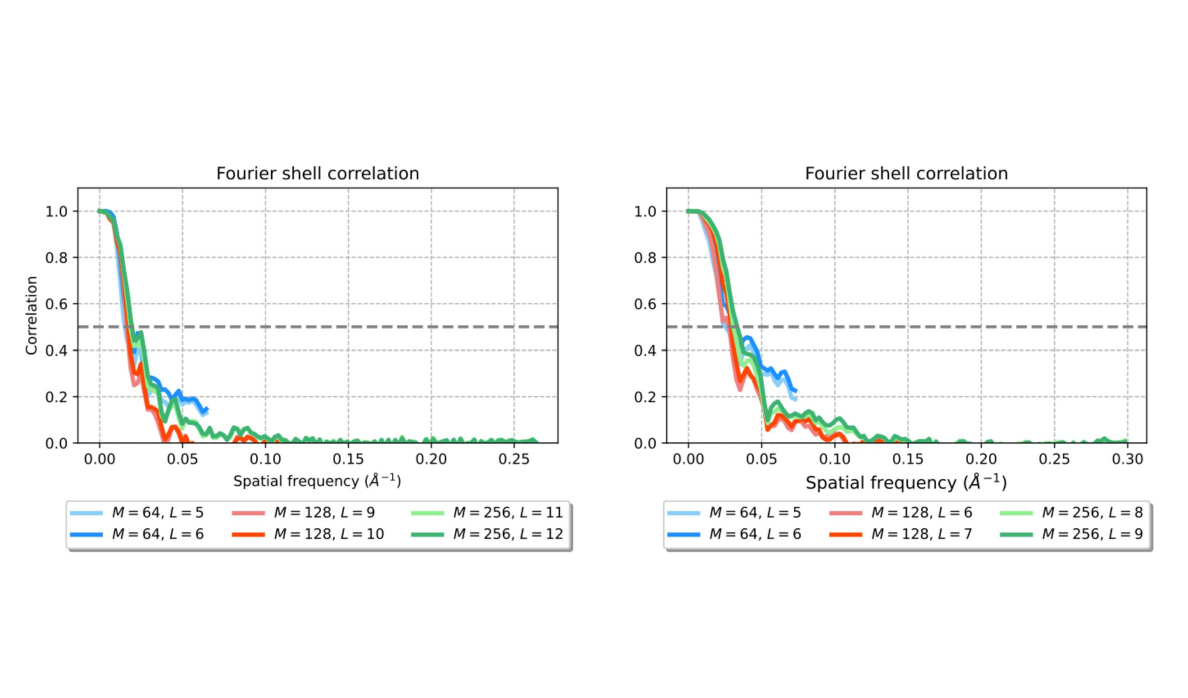}
    \caption{FSC curves between the ground-truth structures and the original structures EMDB-2660 (left) and EMDB-32743 (right).}
    \label{fig:trunc_vs_not_trunc}
\end{figure}

\begin{figure}[!ht]
    \centering
    \includegraphics[width=0.9\linewidth]{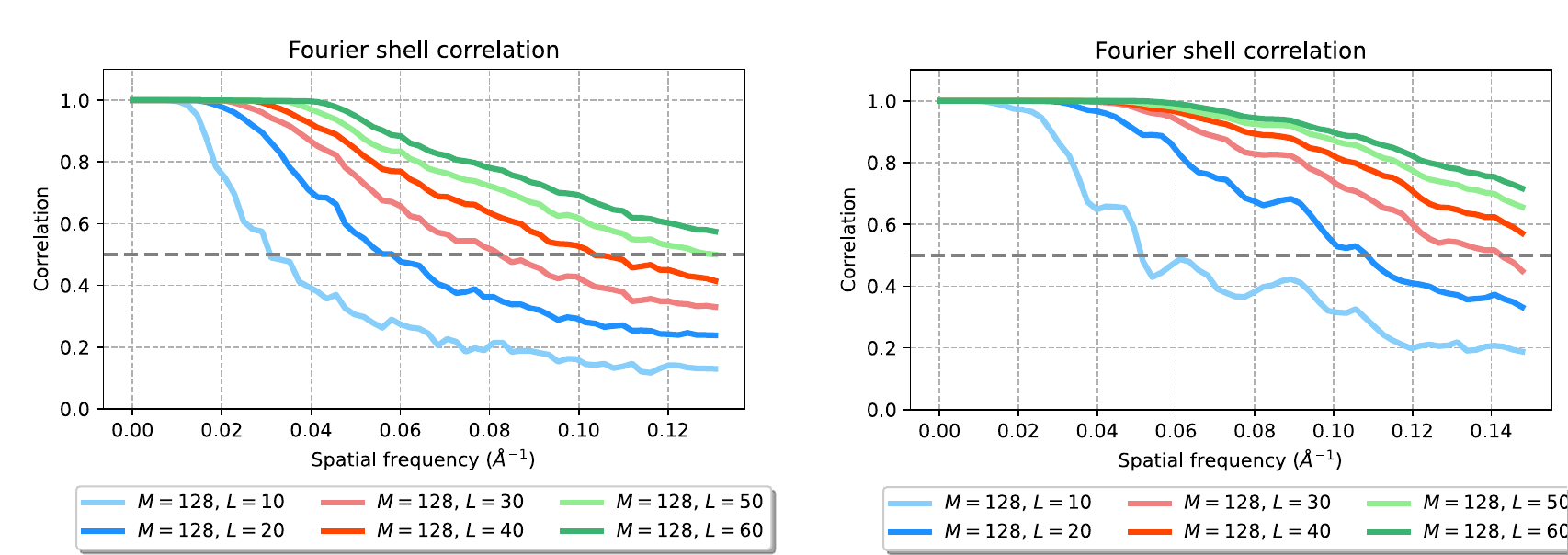}
    \caption{FSC curves between the ground-truth structures and the original structures EMDB-2660 (left) and EMDB-32743 (right), for image size \(M=128\) and varying bandlimit \(L\).}
    \label{fig:truncation_error}
\end{figure}

We emphasize that the current framework is primarily designed to recover the corresponding low-passed structure associated with a prescribed bandlimit. Although increasing \(L\) improves the approximation to the original structures, our current method empirically exhibits a maximal bandlimit that can be stably reconstructed, and this maximal achievable bandlimit depends on the number \(M\) of pixels in the projection images. As  illustrated in the figures above and further shown in Figure~\ref{fig:numerical_results3} below, the achievable bandlimit increases with the image resolution. 
Understanding this dependence theoretically and developing methods that can reliably and efficiently operate at higher bandlimits are important directions for future work.

\begin{figure}[!h]
    \centering
    \includegraphics[width=\linewidth,
    trim={0cm 5cm 0cm 5cm},
    clip]{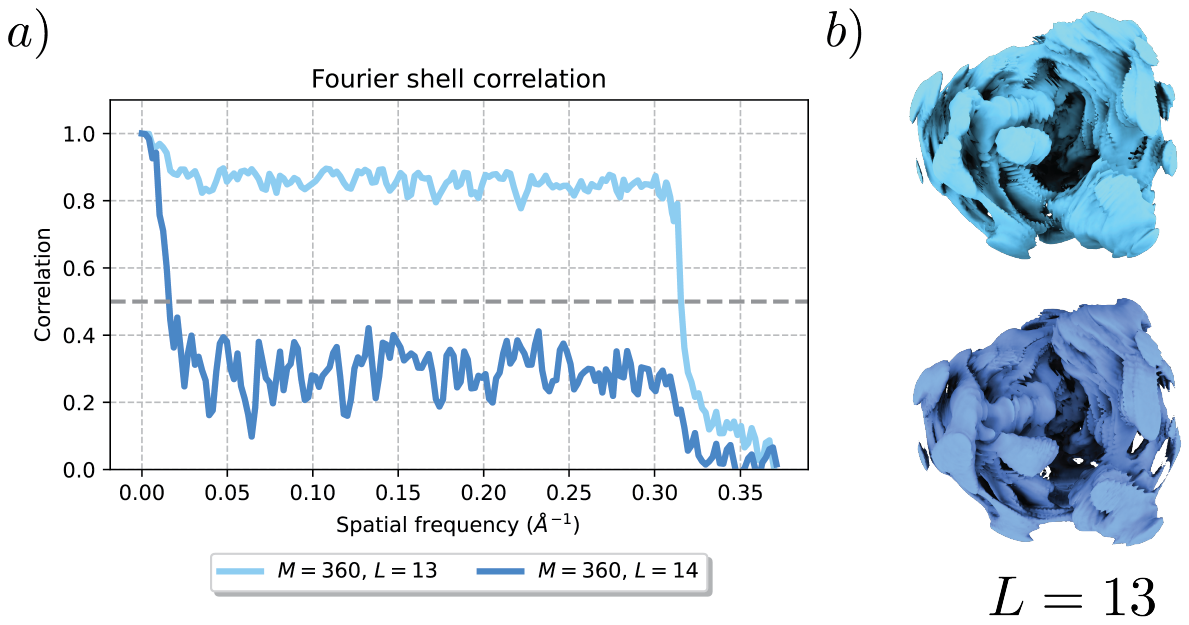}
    \caption{Result of running Algorithm~\ref{alg:modm} using as ground truth structure a bandlimited version of the dataset EMD-2660 \cite{wong2014cryo} from the online Electron Microscopy Data Bank  \cite{lawson2016emdatabank}. (a) FSC curves between the reconstructions and the corresponding ground-truth, for different values of bandlimit \(L\) and image size \(M\). (b) Reconstruction (top) and the corresponding ground-truth structure (bottom).} 
\label{fig:numerical_results3}
\end{figure}

\section{Uniqueness Theorem}\label{sec:uniqueness}

We give the main theoretical result of this paper. 
Theorem~\ref{thm:main} ensures that the first and second population moments of the 2-D images \textit{generically} \textit{uniquely} identify the structure $\Phi$ and the non-uniform distribution $\rho_2$ (i.e., the fourth item of Assumption~\ref{assumption:distributions}), when $\rho_2$ is suitably low-passed. Here, \textit{generic} refers to a generic condition on the distribution $\rho_2$, and \textit{unique} emphasizes the unique recovery of the structure $\Phi$ (in contrast to list-recovery formulations; see, e.g., \cite{sharon2020method, bandeira2023estimation}). 
This theorem provides a fundamental guarantee for the computational problem stated in Problem~\ref{prob:main}. In the statement below, we denote by $\rho_2^{\downarrow 2L}$ the low-pass of $\rho_2$ to degree $2L$,  i.e., \eqref{eq:def_rho_in_plane_uniform} truncated to $p \leq 2L$.

\begin{theorem}\label{thm:main}
   Assume $\Phi$, $\rho_1$, and $\rho_2$ satisfy Assumption~\ref{assumption:distributions}.
    Then population moments $m_1[\Phi, \rho_1]$, $m_2[\Phi,\rho_1]$ and $m_2[\Phi,\rho_2]$ uniquely identify $\Phi$ and $\rho_2^{\downarrow 2L}$ up to the action of $\SO(3)$ on $(\Phi, \rho_2^{\downarrow 2L})$ and up to chirality.
\end{theorem}
The proof of Theorem~\ref{thm:main} is at the end of this section.

\begin{remark}
Theorem~\ref{thm:main} identifies the structure up to an overall global rotation and reflection.
This is an unavoidable ambiguity in cryo-EM, and therefore not a drawback of this particular result.  For clarity, we include Lemma~\ref{lem:symmetry} below which formalizes the ambiguity. 
\end{remark}

\begin{lemma}\label{lem:symmetry}
Let $\Phi, \rho_1, \rho_2$ and $\tilde{\Phi}, \rho_1, \tilde{\rho}_2$ both satisfy satisfy Assumption~\ref{assumption:distributions}.  Assume they differ from each other by a global rotation and possibly chirality,  
that is, there exist $S \in \SO(3)$ and $\epsilon \in \{0,1\}$ such that for all $\x \in \mathbb{R}^3$ and $R \in \SO(3)$ it holds 
\begin{equation}
\tilde{\Phi}(\x) = \Phi(J^{\epsilon} S \x) \quad \quad \text{and} \quad \quad \tilde{\rho}_2(R) = \rho_2(J^{\epsilon}SRJ^{\epsilon}),
\end{equation}
where $J = \diag(1,1,-1)$.  Then we have 
\begin{equation} \label{eq:lemma-want}
m_1[\Phi, \rho_1] = m_1[\tilde{\Phi}, {\rho_1}], \quad  m_2[\Phi, \rho_1] = m_2[\tilde{\Phi}, {\rho_1}], \quad  \text{and}  \quad m_2[\Phi, \rho_2] = m_2[\tilde{\Phi}, \tilde{\rho}_2].
\end{equation}
\end{lemma}

See Appendix~\ref{sec:appendixB} for a proof of Lemma~\ref{lem:symmetry}.

\begin{remark}\label{rem:caveat}
   Our proof strategy for Theorem~\ref{thm:main} goes by induction on $L$.  
   The verification of the base case $L=3$ requires computer assistance.
   In particular, we use floating-point arithmetic and pseudo-random numbers for the base case.  
   Although such calculations are standard practice in the field of computational algebra, they fall short of providing a completely rigorous proof; thus the proof of Theorem~\ref{thm:main} is labeled as a ``Computational Proof" below.  All other steps in the argument are rigorous.
\end{remark}

Our induction step for Theorem~\ref{thm:main} relies on the following lemmas.  The lemmas imply that certain linear systems arising in the proof generically have a unique solution.
For convenience, we denote by $B_p = (B_{p,u})_{-p \le u \le p}$ the vector in $\mathbb{C}^{2p+1}$.

\begin{lemma}
\label{lem:technical1}
Let $P \geq 2L$.  If the expansion coefficients $\left(B_{p} : 2 \leq p \leq 2L-2, p \text{ is even} \right)$ are fixed, then the affine-linear map 
\begin{equation} \label{eq:allBnLL}
B_{2L} \mapsto \left( \mathcal{B}^n_{L,L} : 0 \leq n \leq L, n \equiv L \,\, (\textup{mod } 2) \right)
\end{equation}
is injective.
\end{lemma}

\begin{lemma} 
\label{lem:technical2}
Let $L \geq 4$.   If the expansion coefficients $\left(B_{p} : 2 \leq p \leq 2L-2, p \text{ is even} \right)$ are Zariski-generic, then 
the horizontal concatenation of the matrices
\begin{equation} \label{eq:big-matrix}
\left( \mathcal{B}^{n}_{L, \ell'} : 0 \leq \ell' < L, n \equiv L \equiv \ell' \,\, (\textup{mod }2) \right)
\end{equation}
has full column rank.
\end{lemma} 

See Appendix~\ref{sec:appendixB} for proofs of Lemmas~\ref{lem:technical1} and \ref{lem:technical2}.  
We are now ready to prove our main theoretical result.

\begin{proof}[Computational Proof of Theorem~\ref{thm:main}]
Let $\Phi$, $\rho_1$, and $\rho_2$ satisfy Assumption~\ref{assumption:distributions}.  Suppose $\tilde{\Phi}$ and $\tilde{\rho}_2$ are another structure and distribution with expansions
\begin{equation}
\widehat{\tilde{\Phi}}(r,\theta,\varphi) = \sum_{\ell=0}^L\sum_{m=-\ell}^\ell \tilde{A}_{\ell}^m(r) Y_\ell^m(\theta, \varphi), \quad r \in [0, r_\text{max}], \quad \theta \in [0, \pi], \quad \varphi \in [0, 2\pi),
\end{equation}
and 
\begin{equation}
\tilde{\rho}_2(R) \mathrm{d}R =\sum_{\substack{p=0\\ p \text{ even}}}^{P} \sum_{u=-p}^p \tilde{B}_{p,u} U_{u0}^p(R)\mathrm{d}R,
\end{equation}
such that there is a matching of population moments:
\begin{equation}
m_1[\Phi, \rho_1] = m_1[\tilde{\Phi}, \rho_1], \quad m_2[\Phi, \rho_1] = m_2[\tilde{\Phi}, \rho_1], \quad m_2[\Phi, \rho_2] = m_2[\tilde{\Phi}, \tilde{\rho}_2].
\end{equation}
Our goal is to prove that $(\tilde{\Phi}, \tilde{\rho}_2^{\downarrow 2L})$ equals $(\Phi, \rho_2^{\downarrow 2L})$ up to rotation and possibly chirality. 
As in Lemma~\ref{lem:symmetry}, this precisely means there exist $S \in \SO(3)$ and $\epsilon \in \{0,1\}$ such that $\tilde{\Phi}(\x) = \Phi(J^{\epsilon} S \x)$ and $\tilde{\rho}_2^{\downarrow 2L}(R) = \rho_2^{\downarrow 2L}(J^{\epsilon}SRJ^{\epsilon})$. 
Notice that $\tilde{\Phi}(\x) = \Phi(J^{\epsilon} S \x)$ is equivalent to 
\begin{align}\label{eq:group-action}
\tilde{A}_{\ell}(r) = A_{\ell}(r) U^{\ell}(J^{\epsilon}S),  \quad \text{for all } 0\leq\ell \leq L,
\end{align}
because, writing $\x = (r,\theta,\varphi)$ in spherical coordinates, we have
\[
\widehat{\Phi}(J^{\epsilon}S\x) = \sum_{\ell=0}^L \sum_{m=-\ell}^{\ell} A_{\ell m}(r) \sum_{n=-\ell}^\ell U^{\ell}_{m n}(J^{\epsilon} S) Y^{n}_{\ell}(\theta,\varphi) 
= \sum_{\ell=0}^L \sum_{m=-\ell}^{\ell} [A_{\ell}(r) U^{\ell}(J^{\epsilon}S)]_{m} Y^{m}_{\ell}(\theta, \varphi),
\]
where $A_{\ell}(r)$ denotes the row-vector-valued function evaluated at the radial frequency $r$, as defined in \eqref{eq:def_Al}.

Toward this goal, we first use the condition $m_2[\Phi, \rho_1] = m_2[\tilde{\Phi}, \rho_1]$.  By \cite{kam1980reconstruction},  it implies 
\begin{equation}\label{eq:kam-O}
\tilde{A}_{\ell}(r) =  A_{\ell}(r) Q_{\ell}^{\mathsf{H}} {O}_{\ell}^{\T}Q_{\ell}, \quad \text{for all } 0 \leq \ell \leq L,
\end{equation}
where ${O}_{\ell} \in \mathbb{R}^{(2 \ell+1) \times (2\ell+1)}$ are some unknown real-valued orthogonal matrices and $Q_{\ell} \in \mathbb{C}^{(2 \ell +1) \times (2 \ell +1)}$ are the complex-valued unitary matrices defined in \eqref{eq:Q-def}, which represent the unitary transforms from the complex to the real spherical harmonics basis (cf.~\eqref{eq:Aring_relation}).
Next, we use $m_1[\Phi, \rho_1] = m_1[\tilde{\Phi}, \rho_1]$.  This implies $\tilde{A}_0(r) = A_0(r)$,  again by \cite{kam1980reconstruction}.  Since  $Q_{0}^{\mathsf{H}} =1$ it follows that 
\begin{equation} \label{eq:O0}
O_0=1,
\end{equation}
where $O_0 \in \mathbb{R}$ is the $1 \times 1$ orthogonal matrix in \eqref{eq:kam-O}.

Next, we claim that by the symmetry in \eqref{eq:group-action}, it is without loss of generality to assume  
\begin{equation} \label{eq:O1}
O_1 = I \in \mathbb{R}^{3 \times 3}.
\end{equation}
Indeed, if $O_1 \neq I$ there exist  $S \in \SO(3)$ and $\epsilon \in \{0,1\}$ such that $Q_1^{\mathsf{H}}U^{1}(J^{\epsilon}S)Q_1 = O_1$ (see \cite[pg.~324]{chirikjian2016harmonic}). 
Then the replacement $\Phi(\x) \leftarrow \Phi(J^{\epsilon}S\x)$ reduces us to $O_1 = I$ (cf. \eqref{eq:kam-O}).  Note that assuming $O_1 = I$ kills rotational and chiral ambiguities of the problem, i.e., the goal becomes to prove that these equalities hold exactly, $\tilde{\Phi} = \Phi$ and $\tilde{\rho}_2 = \rho_2$. Thus with \eqref{eq:O1}, we want

\begin{align}
&O_{\ell}=I\in\R^{(2\ell+1)\times(2\ell+1)}, \quad \text{ for all } 0 \leq \ell \leq L, \text{ and}  \label{eq:want-Ol}\\  &\tilde{B}_{p,u} = B_{p,u}, \quad \text{ for all } 0 \leq p \leq 2L \text{ with }p \text{ even and}-p \leq u \leq p. \label{eq:want-Btilde}
\end{align}

To show \eqref{eq:want-Ol} and \eqref{eq:want-Btilde}, we turn to the condition $m_2[\Phi, \rho_2] = m_2[\tilde{\Phi}, \tilde{\rho}_2]$. 
First by \eqref{eq:2nd_moment_matrix},
\begin{align} \label{eq:m2Phi}
&m_2[\Phi, \rho_2](r,\varphi,r',\varphi') = \sum_{n=-L}^{L} e^{\ii n(\varphi-\varphi')} \sum_{\ell = 0}^L \sum_{\ell'=0 }^L   A_{\ell}(r) \mathcal{B}^n_{\ell,\ell'} (A_{\ell'}(r'))^{\mathsf{H}}.
\end{align}
Combining with \eqref{eq:kam-O},
\begin{align}\label{eq:m2Phi'}
&m_2[\tilde{\Phi}, \tilde{\rho}_2](r,\varphi,r',\varphi') = \sum_{n=-L}^{L} e^{\ii n(\varphi-\varphi')} \sum_{\ell=0}^L  \sum_{\ell'=0 }^L   A_{\ell}(r) Q_{\ell}^{\mathsf{H}} O_{\ell}^{\T} Q_{\ell} \tilde{\mathcal{B}}^n_{\ell,\ell'} Q_{\ell'}^{\mathsf{H}} O_{\ell'} Q_{\ell'} (A_{\ell'}(r'))^{\mathsf{H}}.
\end{align}
Equating \eqref{eq:m2Phi} and \eqref{eq:m2Phi'} and using orthonormality of the Fourier modes deduce
\begin{align} \label{eq:equate1}
\sum_{\ell = 0}^L \sum_{\ell'=0 }^L   A_{\ell}(r) \mathcal{B}^n_{\ell,\ell'} (A_{\ell'}(r'))^{\mathsf{H}} = \sum_{\ell=0}^L  \sum_{\ell'=0 }^L   A_{\ell}(r) Q_{\ell}^{\mathsf{H}} O_{\ell}^{\T} Q_{\ell} \tilde{\mathcal{B}}^n_{\ell,\ell'} Q_{\ell'}^{\mathsf{H}} O_{\ell'} Q_{\ell'} (A_{\ell'}(r'))^{\mathsf{H}}
\end{align}
for each $n$ satisfying $-L \leq n \leq L$.  By the assumed linear independence of the radial functions $A_{\ell}^m$ (i.e., the second item of Assumption~\ref{assumption:distributions}), this implies 
\begin{align}\label{eq:key}
\mathcal{B}^n_{\ell,\ell'} = Q_{\ell}^{\mathsf{H}} O_{\ell}^{\T} Q_{\ell} \tilde{\mathcal{B}}^n_{\ell,\ell'} Q_{\ell'}^{\mathsf{H}} O_{\ell'} Q_{\ell'}
\end{align}
for each $(n, \ell, \ell')$ with $-L \leq n \leq L$ and $0 \leq \ell, \ell' \leq L$.  Here $\tilde{\mathcal{B}}^n_{\ell, \ell'}$ depends on $\tilde{B}_p$ as ${\mathcal{B}}^n_{\ell, \ell'}$ does on $B_p$.

In the rest of the proof, our strategy is to leverage \eqref{eq:key} for $(n, \ell, \ell')$ in an appropriate order to establish \eqref{eq:want-Ol} and \eqref{eq:want-Btilde} by induction on $L$.  Precisely, we induct on the following:
\begin{align}
&\textit{Take Assumption~\ref{assumption:distributions}.  Then \eqref{eq:key} for}  -L \leq n \leq L \textup{ and } 0 \leq \ell, \ell' \leq L, \textit{ together with} \nonumber \\ &\textit{\eqref{eq:O0} and \eqref{eq:O1}, imply  } O_{\ell} = I \textit{ for } 0 \leq \ell \leq L \textit{ and } \tilde{B}_{p} = B_p \textit{ for } p = 0, 2, \ldots, 2L. \label{eq:induction}
\end{align}
The induction amounts to showing that a polynomial system in  $O_{\ell}$ and $\tilde{B}_{p}$ has a unique solution.

The base case of the induction is $L=3$.  
By \eqref{eq:constraints_on_B}, \eqref{eq:O0} and \eqref{eq:O1}, we know $B_0 = \tilde{B}_0$, $O_0 = 1$ and $O_1 = I$.  We wish to show $\tilde{B}_2 = B_2$, $\tilde{B}_{4} = B_4$, $\tilde{B}_{6} = B_6$, $O_2 = I$ and $O_3 = I$.  
We will use equations \eqref{eq:key} suitably rearranged, together with orthogonality constraints:
\begin{equation} \label{eq:big-polynomial}
\begin{cases}
& Q_{\ell}^{\mathsf{H}} O_{\ell} Q_{\ell} \mathcal{B}^n_{\ell,\ell'} Q_{\ell'}^{\mathsf{H}} O_{\ell'}^{\T} Q_{\ell'} =  \tilde{\mathcal{B}}^n_{\ell,\ell'}  \\
& O_2 O_2^{\T} = O_2^{\T} O_2 = I \\ 
& O_3 O_3^{\T} = O_3^{\T} O_3 = I,
\end{cases}
\end{equation}
where in the first line $(n, \ell, \ell') \in \{(1,1,1), (0,2,0), (0,2,2), (1,3,1), (1,3,3), (3,3,3)\}$.
Dropping reality constraints on $O_{\ell}$ and $\tilde{B}_p$, we view \eqref{eq:big-polynomial} as a parameterized polynomial system over $\mathbb{C}$: the variables are $(\tilde{B}_2, \tilde{B}_4, \tilde{B}_6, O_2, O_3) \in \mathbb{C}^{101}$, the parameters are $(B_2, B_4, B_6) \in \mathbb{C}^{27}$, and there are $(3 \times 3) + (5 \times 1) + (5 \times 5) + (5 \times 5) + (7 \times 3) + (7 \times 7) + (7 \times 7) + (5 \times 5) + (5 \times 5) + (7 \times 7) + (7 \times 7) = 331$ equations.  Note the equations are affine-quadratic or linear in the variables and affine-linear in the parameters.
By general properties of parameterized polynomial systems over $\mathbb{C}$, there exists a nonempty Zariski-open subset $\mathcal{U} \subseteq \mathbb{C}^{27}$ such that for $(B_2, B_4, B_6) \in \mathcal{U}$ the solution set to \eqref{eq:big-polynomial} has the same ``type" of irreducible decomposition in the sense of \cite[Theorem~A.14.10]{sommese2005numerical}. 
Therefore, if we show on a randomly-generated instance of $(B_2, B_4, B_6)$ that the polynomial system \eqref{eq:big-polynomial} has a unique solution over $\mathbb{C}$, then with probability 1 the system Zariski-generically has a unique solution over $\mathbb{C}$, which then must be the trivial solution $\tilde{B}_p=B_p$ for $p=2,4,6$ and $O_{\ell} = I$ for $\ell = 2, 3$.  Checking polynomial systems on random instances is a standard approach in computational algebra; that said, Remark~\ref{rem:caveat} applies.  

Here we perform the check using the numerical homotopy continuation and computer algebra software \cite{breiding2018homotopycontinuation, M2}.
We generate $(B_2, B_4, B_6)$ using a random number generator.  
The system \eqref{eq:big-polynomial} is too big to directly input into the software, so we break up the computation.  
First, using the top line of \eqref{eq:big-polynomial} when $(n, \ell, \ell') = (1,1,1)$ it follows by a linear solve or by Lemma~\ref{lem:technical1} that $\tilde{B}_{2} = {B}_2$.  Secondly,  we take $(n, \ell, \ell') = (0,2,0)$, which gives the equation $Q_{2}^{\mathsf{H}} O_{2} Q_{2} \mathcal{B}^{0}_{2,0} = \mathcal{B}^{0}_{2,0}$ or 
\begin{equation} \label{eq:reduceto4}
O_{2} Q_{2} \mathcal{B}^{0}_{2,0} = Q_2 \mathcal{B}^{0}_{2,0},
\end{equation}
i.e., $O_{2}$ fixes a known vector.  We find that the vector is non-isotropic, i.e., $(Q_2 \mathcal{B}^{0}_{2,0})^{\T} (Q_2 \mathcal{B}^{0}_{2,0}) \neq 0$, so we can extend the vector suitably scaled to a complex orthogonal matrix, i.e., find $\tilde{O}_2^{\T} \in \O(5, \mathbb{C})$ and $\lambda \in \mathbb{C}$ such that the leftmost column of $\tilde{O}_2^{\T}$ is $\lambda Q_2 \mathcal{B}^{0}_{2,0}$.  Then $\tilde{O}_2^{\T}e_1=\lambda Q_2 \mathcal{B}^{0}_{2,0}$ or $\lambda \tilde{O}_2 Q_2 \mathcal{B}^{0}_{2,0} = e_1$, where $e_1$ is the first standard basis, and \eqref{eq:reduceto4} can be rewritten as $(\tilde{O}_2 O_{2} \tilde{O}_2^{\T}) \lambda \tilde{O}_2 Q_{2} \mathcal{B}^{0}_{2,0} = \lambda \tilde{O}_2 Q_2 \mathcal{B}^{0}_{2,0}$ or 
\begin{equation*}
(\tilde{O}_2 O_{2} \tilde{O}_2^{\T}) e_1 = e_1.
\end{equation*}
Using $(\tilde{O}_2 O_{2} \tilde{O}_2^{\T})^{\T}(\tilde{O}_2 O_{2} \tilde{O}_2^{\T}) = I$, it follows that $\tilde{O}_2 O_{2} \tilde{O}_2^{\T} = 1 \oplus \tilde{\tilde{o}}_2$ for some $\tilde{\tilde{o}}_2 \in \O(4,\mathbb{C})$, i.e., $\tilde{O}_2 O_{2} \tilde{O}_2^{\T}$ is block-diagonal.  Next, consider $(n, \ell, \ell') = (0,2,2), (2,2,2)$ in \eqref{eq:big-polynomial}:
\begin{equation} \label{eq:lineelim}
Q_{2}^{\mathsf{H}} (1 \oplus \tilde{\tilde{o}}_2) Q_{2} \mathcal{B}^n_{2,2} Q_{2}^{\mathsf{H}} (1 \oplus \tilde{\tilde{o}}_2)^{\T} Q_{2} =  \tilde{\mathcal{B}}^n_{2,2}, 
\end{equation}
for $n=0,2$.  The right-hand side depends affine-linearly on $\tilde{B}_4$ and on no other unknowns.  Therefore, we can linearly eliminate $\tilde{B}_4$ from \eqref{eq:lineelim}.  We find $5 \times 5 - 9 = 16$ affine-quadratic equations in $\tilde{\tilde{o}}_2$ for $n=0,2$.  We solve the resulting polynomial system of these 36 equations together with $\tilde{\tilde{o}}_2 \tilde{\tilde{o}}_2^{\T} = \tilde{\tilde{o}}_2^{\T} \tilde{\tilde{o}}_2 = I$
in the variables $\tilde{\tilde{o}}_2 \in \mathbb{C}^{16}$ using the software \cite{breiding2018homotopycontinuation}.  Four isolated multiplicity-1 solutions are computed: $\tilde{\tilde{o}}_2^{(1)}, \tilde{\tilde{o}}_2^{(2)}, \tilde{\tilde{o}}_2^{(3)}, \tilde{\tilde{o}}_2^{(4)} \in \mathbb{C}^{16}$.  We  return to \eqref{eq:lineelim}, substitute in these solutions for $\tilde{\tilde{o}}_2$ and linearly solve for $\tilde{B}_4 \in \mathbb{C}^9$.  Corresponding to each $\tilde{\tilde{o}}_2^{(i)}$ we find a unique solution $\tilde{B}_{4}^{(i)} \in \mathbb{C}^{9}$.  Next up,  use \eqref{eq:big-polynomial} with $(n,\ell,\ell') = (1,3,1)$:
\begin{equation}\label{eq:OKAY}
O_{3} Q_{3} \mathcal{B}^1_{3,1}  = Q_3 \tilde{\mathcal{B}}^1_{3,1}. 
\end{equation}
Left-multiplying each side by its transpose yields
\begin{equation}\label{eq:closing-in}
(\mathcal{B}^1_{3,1})^{\T} Q_3^{\T} Q_3 \mathcal{B}^1_{3,1} =  (\tilde{\mathcal{B}}^1_{3,1})^{\T} Q_3^{\T} Q_3  \tilde{\mathcal{B}}^1_{3,1}. 
\end{equation}
To evaluate $\tilde{\mathcal{B}}^{1}_{3,1}$ we plug in the possible values for $\tilde{B}_{4}$, namely  $\tilde{B}_{4}^{(1)}, \tilde{B}_{4}^{(2)}, \tilde{B}_{4}^{(3)}, \tilde{B}_{4}^{(4)}$, and find that only one of these satisfies \eqref{eq:closing-in}.  It is the value that equals $B_4$, and corresponds to $\tilde{\tilde{o}}_2 = I$ or $\tilde{O}_2 O_2 \tilde{O}_2^{\T} = I$.
We conclude $\tilde{B}_4 = B_4$ and $O_2 = I$.  Next, we reuse \eqref{eq:OKAY}: since $\tilde{\mathcal{B}}^1_{3, 1} = \mathcal{B}^1_{3,1}$, it says that $O_3$ fixes a known $7 \times 3$ matrix. 
Similarly to how we utilized \eqref{eq:reduceto4}, here we find $\tilde{O}_3 \in \O(7, \mathbb{C})$ such that the first three columns of $\tilde{O}_3^{\T}$ span the column space of $Q_1 \tilde{\mathcal{B}}^1_{3,1}$ and $\tilde{O}_3 O_3 \tilde{O}_3^{\T} = I_3 \oplus \tilde{\tilde{o}}_3$ for some $\tilde{\tilde{o}}_3 \in \O(4, \mathbb{C})$ where $I_3$ denotes the $3 \times 3$ identity matrix.  Then, consider \eqref{eq:big-polynomial} with $(n, \ell, \ell') = (1,3,3), (3,3,3)$:
\begin{equation} \label{eq:lineelim2}
Q_{3}^{\mathsf{H}} (I_3 \oplus \tilde{\tilde{o}}_3) Q_{3} \mathcal{B}^n_{3,3} Q_{3}^{\mathsf{H}} (I_3 \oplus \tilde{\tilde{o}}_3)^{\T} Q_{3} =  \tilde{\mathcal{B}}^n_{3,3}, 
\end{equation}
for $n=1,3$.  Similarly to how we dealt with \eqref{eq:lineelim}, we linearly eliminate $\tilde{B}_6$ from \eqref{eq:lineelim2}. 
We find $7 \times 7 - 13 = 36$ affine-quadratic equations in $\tilde{\tilde{o}}_3$ for $n=1,3$.  The polynomial system of these $72$ equations with $\tilde{\tilde{o}}_3 \tilde{\tilde{o}}_3^{\T} = \tilde{\tilde{o}}_3^{\T} \tilde{\tilde{o}}_3 = I$ in variables $\tilde{\tilde{o}}_3 \in \mathbb{C}^{16}$ is solved using \cite{breiding2018homotopycontinuation}.  The software computes a unique solution, which is multiplicity-1: $\tilde{\tilde{o}}_3 = I$.  
Hence   $\tilde{O}_3 O_3 \tilde{O}_3^{\T} = I$, or $O_3 = I$.  We return to \eqref{eq:lineelim2} which now reads $\mathcal{B}^{n}_{3,3} = \tilde{\mathcal{B}}^n_{3,3}$ for $n=1,3$.  Linearly solving for $\tilde{B}_6$ or using Lemma~\ref{lem:technical1}, deduce $\tilde{B}_6 = B_6$.  This wraps up the base case of \eqref{eq:induction}. 

Let us now turn to the induction step.  
Thus assume $L \geq 4$, and that \eqref{eq:induction} has been shown for $L-1$.  
We wish to show \eqref{eq:induction} for $L$.  
By the inductive hypothesis, we know $O_{\ell}=I$ for $0 \leq \ell \leq L-1$ and $B_p = \tilde{B}_p$ for $p = 0, 2, \ldots, 2(L-1)$.   We want  $O_L = I$ and $\tilde{B}_{2L} = B_{2L}$.
We will rely on two lemmas concerning the dependence of $\mathcal{B}^n_{\ell, \ell'}$ on $B_{p}$, Lemmas~\ref{lem:technical1} and \ref{lem:technical2}.

Consider \eqref{eq:key} for  $(n, L, \ell')$  satisfying $0 \leq \ell' < L, n \equiv L \equiv \ell' \,\, (\textup{mod }2)$.  For these triples $\mathcal{B}^{n}_{L,\ell'} = \tilde{\mathcal{B}}^{n}_{L,\ell'}$, because $(\mathcal{B}^{n}_{L,\ell'})$ depend only on $(B_p)$ with $p \leq 2L-2$, $(\tilde{\mathcal{B}}^{n}_{L,\ell'})$ depend only on $(\tilde{B}_p)$ with $p \leq 2L-2$ in the same way, and  $B_{p} = \tilde{B}_p$ for $p \leq 2L-2$.  Therefore $\eqref{eq:key}$ reads
\begin{equation*}
\mathcal{B}^{n}_{L,\ell'} = Q_{L}^{\mathsf{H}} O_{L}^{\T} Q_{L} \mathcal{B}^n_{L,\ell'}, 
\end{equation*}
where we used $O_{\ell'}=I$.  So left multiplication by $Q_{L}^{\mathsf{H}} O_L^{\T} Q_{L}$ fixes the concatenation of matrices:
\begin{equation*} \label{eq:my-concat}
\left( \mathcal{B}^{n}_{L, \ell'} : 0 \leq \ell' < L, n \equiv L \equiv \ell' \,\, (\textup{mod }2) \right)\!.
\end{equation*}
By Lemma \ref{lem:technical2}, the concatenation has full column rank.  Thus $Q_{L}^{\mathsf{H}} O_L^{\T} Q_{L} = I$, or $O_{L} = I$. 

Next, consider \eqref{eq:key} for $(n,L,L)$ where $0 \leq n \leq L$ and $n \equiv L \,\, (\textup{mod 2})$.  These read:
\begin{equation*}\label{eq:B-Btilde}
\mathcal{B}^{n}_{L,L} = \tilde{\mathcal{B}}^n_{L,L}, 
\end{equation*}
using $O_{L} = I$.  
As $B_p = \tilde{B}_p$ for $p \leq 2L-2$,  Lemma~\ref{lem:technical1} gives $\tilde{B}_{2L} = B_{2L}$.  This finishes the induction step for \eqref{eq:induction}. 

Putting everything together, we conclude that under Assumption~\ref{assumption:distributions}, the moments $m_1[\Phi, \rho_1]$,  $m_2[\Phi,\rho_1]$, $m_2[\Phi,\rho_2]$ uniquely identify 
$\Phi$ and $\rho_2^{\downarrow 2L}$, 
up to rotation and chirality.
\end{proof}

\begin{remark}
We stress that although the details of Theorem~\ref{thm:main} are tailored to the image formation model of cryo-EM, we envision that similar results will hold for related imaging techniques like XFEL. Similarly, the Algorithm~\ref{alg:modm} is likely adaptable to other settings.
\end{remark}

\section{Stability Analysis}\label{sec:stability}

The algorithm and theory above assume that the first dataset has uniformly distributed viewing directions. In practice, this distribution may only be approximately uniform. This section quantifies the population-level effect of such a model mismatch on the first step of Algorithm~\ref{alg:modm}, namely the Kam's method in \eqref{eq:def_Kam}--\eqref{eq:Aring_relation}. The goal here is not to provide an end-to-end perturbation analysis of the full alternating procedure, but rather to isolate how a deviation from uniformity changes the matrices whose Cholesky factors are used to initialize the recovery of the orthogonal matrices. The proofs are given in Appendix~\ref{sec:appendixC}.

Recall that, under the uniform distribution, the Legendre projection of the second-order moment produces the matrix-valued kernel
\(C_\ell(r,r')\), which equals \(A_\ell A_\ell^{\mathsf{H}}\) after discretization in the radial variable. For a general viewing distribution \(\rho\), we define the analogous quantity by applying the same Legendre projection:
\begin{align}\label{eq:def_Kam_general}
C_\ell[\Phi,\rho](r,r') \!&:=\! 2\pi(2\ell+1)\!\cdot\!\int_0^{\pi} m_2[\Phi, \rho](r, r', \psi) P_\ell(\cos \psi) \sin (\psi )\der \psi,\qquad \ell=0,\ldots,L.
\end{align}
For the uniform distribution $\rho_1^*$, this recovers the Kam matrix
\begin{align*}
C_{\ell}(r,r') =C_\ell[\Phi,\rho_1^*](r,r').
\end{align*}
Given radial sampling points \(\{r_1,\ldots,r_{M_r}\}\), we write \(C_\ell[\Phi,\rho]\in\mathbb{C}^{M_r\times M_r}\) for the sampled matrix with entries
\((C_\ell[\Phi,\rho])_{i,j}=C_\ell[\Phi,\rho](r_i,r_j)\). We measure the bias induced by non-uniformity through
\[
\Delta_\ell(\rho):=C_\ell[\Phi,\rho]-C_\ell[\Phi,\rho_1^\ast].
\]
The following proposition makes explicit two important facts: the bias is linear in the non-uniform coefficients of \(\rho\), and only the coefficients up to degree \(2L\) can affect the second-order moment used in the algorithm. 

To state the result, we introduce a response matrix for each distribution coefficient. Following the notation in Section~2.3, for \(0\leq a,b\leq L\), \(0\leq p\leq \min\{P,2L\}\), \(-p\leq u\leq p\), and \(-L\leq n\leq L\), define \(E^n_{a,b;\, p,u}\in\mathbb{C}^{(2a+1)\times(2b+1)}\) by
\begin{align}\label{eq:def_E}
(E_{a,b;\, p,u}^{n})_{m,m'}
:=
\1_{\{u=m'-m\}}
\1_{\{|a-b|\leq p\leq a+b\}}
\cdot
(-1)^{m+n}
\mathcal{N}^n_{a}
\mathcal{N}^n_{b}
\frac{
\mathcal{C}_{p}
(a,b,m,-m',n,-n)
}{2p+1}.
\end{align}
For $\ell=0,\ldots,L$ and $-L\leq n\leq L$, set
\begin{align}\label{eq:def_alpha}
\alpha_{\ell}^n:=2\pi(2\ell+1)\!\cdot\! \int_0^{\pi}e^{\ii n\psi}P_\ell(\cos \psi) \sin (\psi )d\psi.
\end{align}

\begin{proposition}\label{prop:Kam_mat_stability}
Assume $\Phi$ is bandlimted with respect to its angular variables, i.e., $\Phi$ can be expressed as in \eqref{eq:expand_phi_hat_sph_bessel} for some $L \geq 3$. Let $\rho$ satisfy Assumption~\ref{assumption:distributions}(4); in particular, assume that $\rho$ is bandlimited as in \eqref{eq:def_rho_in_plane_uniform} with cutoff $P\geq 2L$. Then for each $0\leq \ell\leq L$,
\begin{align*}
C_\ell[\Phi,\rho]=\sum_{p=0}^{2L}\sum_{u=-p}^p B_{p,u} H_{\ell;p,u}[\Phi],
\end{align*}
where \(H_{\ell;\, p,u}[\Phi]\in\C^{M_r\times M_r}\) is the response matrix of the $\ell$-th Kam block to the distribution coefficient $B_{p,u}$, whose \((i,j)\)-th entry is
\begin{align}\label{eq:def_Kam_response}
(H_{\ell;\, p,u}[\Phi])_{i,j}\equiv H_{\ell;\, p,u}[\Phi](r_i,r_j):= \sum_{n=-L}^{L} \alpha_{\ell}^n \sum_{a=0}^L  \sum_{b=0}^L   A_{a}(r_i) E_{a,b;p,u}^{n} (A_{b}(r_j))^{\mathsf{H}}.
\end{align}
Consequently, with
\begin{align*}
\beta_{2L}(\rho):=& \|\rho-\rho_1^*\|_{L_2(\SO(3))}=\Big(\sum_{p=1}^{2L}\sum_{u=-p}^p \frac{|B_{p,u}|^2}{2p+1} \Big)^{1/2},\\
\Gamma_{\ell}(\Phi):=& \Big(\sum_{p=1}^{2L}\sum_{u=-p}^p (2p+1)\cdot \|H_{\ell;\,p,u}(\Phi)\|_F^2 \Big)^{1/2},
\end{align*}
we have
\begin{align}\label{eq:perturbation_Kam}
\|\Delta_\ell(\rho)\|_F\leq \Gamma_{\ell}(\Phi)\beta_{2L}(\rho).
\end{align}
\end{proposition}
Proposition~\ref{prop:Kam_mat_stability} shows that the population-level bias in the Kam matrix is controlled linearly by the deviation of the viewing distribution from uniformity. The constant \(\Gamma_\ell(\Phi)\) depends on the structure and the radial grid, while \(\beta_{2L}(\rho)\) measures precisely the part of the distribution that is visible to the second-order moment at bandlimit \(2L\). 

We next translate this perturbation of the Kam matrix into a perturbation of its factor. For the uniform distribution,
\[
C_\ell[\Phi,\rho_1^\ast]=A_\ell A_\ell^{\mathsf{H}}.
\]
For non-uniform \(\rho\), however, \(C_\ell[\Phi,\rho]\) need not be positive semidefinite; see Remark~\ref{rmk:counterexample} below. We therefore first project \(C_\ell[\Phi,\rho]\) onto the cone of positive semidefinite matrices of rank at most \(2\ell+1\), and then take a Cholesky factorization.  Since right unitary transformations give equivalent
factorizations, we therefore measure the discrepancy between two factors $A_\ell,A_\ell'\in\C^{M_r\times (2\ell+1)}$ modulo this intrinsic ambiguity by defining
\begin{align}\label{eq:dist_unitary}
\operatorname{dist}_F(A_\ell',A_\ell)
:=
\min_{U\in\mathsf{U}(2\ell+1)}
\|A_\ell'-A_\ell U\|_F.
\end{align}
For a Hermitian matrix \(M\in\mathbb{C}^{M_r\times M_r}\) with eigendecomposition
\(M=V\diag(\lambda_1,\ldots,\lambda_{M_r})V^{\mathsf{H}}\), \(\lambda_1\geq\cdots\geq\lambda_{M_r}\), let
\[
\mathcal{P}^+_{2\ell+1}(M):=V\diag\bigl((\lambda_1)_+,\ldots,(\lambda_{2\ell+1})_+,0,\ldots,0\bigr)V^{\mathsf{H}},
\qquad (t)_+:=\max\{t,0\},
\]
denote the best positive semidefinite approximation of rank at most \(2\ell+1\) in Frobenius norm.
\begin{theorem}\label{thm:stability}
Fix $0 \le \ell \le L$ and assume that
$A_\ell \in \mathbb{C}^{M_r \times (2\ell+1)}$ has full column rank with
\[
s_\ell := \sigma_{\min}(A_\ell)>0.
\]
Let $A_\ell(\rho)\in \mathbb{C}^{M_r\times (2\ell+1)}$ be any factor satisfying
\[
A_\ell(\rho) (A_\ell(\rho))^{\mathsf H}
=
\mathcal P^+_{2\ell+1}\bigl(C_\ell[\Phi,\rho]\bigr),
\]
with zero columns appended if the rank is smaller than $2\ell+1$, where
$\mathcal P^+_{2\ell+1}\bigl(C_\ell[\Phi,\rho]\bigr)$ denotes the best
positive semidefinite approximation to
$C_\ell[\Phi,\rho]$ of rank at most $2\ell+1$ in Frobenius norm. Then
\[
\operatorname{dist}_F(A_\ell(\rho),A_\ell)
\le \frac{2c_0}{s_\ell}\|\Delta_\ell(\rho)\|_F,
\qquad
c_0 := \frac{1}{\sqrt{2(\sqrt2-1)}}.
\]
Under the conditions of Proposition~\ref{prop:Kam_mat_stability}, we further have
\[
\operatorname{dist}_F(A_\ell(\rho),A_\ell)
\le
\frac{2c_0}{s_\ell}\,
\Gamma_\ell(\Phi)\,
\beta_{2L}(\rho).
\]
\end{theorem}

Theorem~\ref{thm:stability} gives a population-level stability estimate for the Cholesky/Kam factors used in the first step of the algorithm. The bound has two natural components. First, the error scales linearly with the degree-\(2L\) deviation of the viewing distribution from uniformity. Second, it is inversely related to \(s_\ell=\sigma_{\min}(A_\ell)\), so poorly conditioned Kam blocks amplify the effect of any distributional mismatch. This provides a quantitative interpretation of the requirement that the primary distribution be close to uniform: it should be close enough that the induced bias is small relative to the conditioning of the corresponding Kam factor.

\begin{remark}\label{rmk:counterexample}
The positive semidefinite projection in Theorem~\ref{thm:stability} is not merely a technical convenience. For a non-uniform viewing distribution, the matrices \(C_\ell[\Phi,\rho]\) need not be positive semidefinite. To see this, consider $\ell=0$ and a single radial sampling point $r_\star$. Let
\[
\rho_\varepsilon(R)=1+\varepsilon U^2_{00}(R),
\]
so that $B_{0,0}=1$, $B_{2,0}=\varepsilon$, and all other $B_{p,u}$ vanish.
For sufficiently small $\varepsilon>0$, this is a nonnegative in-plane uniform and
chirality-invariant density. Choose
\[
A_0^0(r_\star)=\delta_1\neq 0,\qquad
A_2^2(r_\star)=A_2^{-2}(r_\star)=\frac{\delta_2}{\sqrt2}\neq 0,
\]
and set all other coefficients equal to zero at $r_\star$. Then the uniform part gives
\[
C_0[\Phi,\rho_1^\ast](r_\star,r_\star)=\delta_1^2.
\]
On the other hand, direct substitution into the definition of $E_{a,b;\, p,u}^n$ gives
\begin{align*}
\sum_{n=-2}^n \alpha_0^n (-1)^{m+n} (\mathcal{N}^n_{2})^2\frac{
\mathcal{C}_{p}
(2,2,m,-m,n,-n)
}{5}=-\frac{1}{7},\quad\text{for } m=2,-2.
\end{align*}
Therefore
\[
H_{0;2,0}[\Phi](r_\star,r_\star)
=
-\frac{\delta_2^2}{7},
\]
and hence
\[
C_0[\Phi,\rho_\varepsilon](r_\star,r_\star)
=
\delta_1^2-\varepsilon\frac{\delta_2^2}{7}.
\]
Choosing $\delta_2>\sqrt{7\delta_1^2/\varepsilon}$ makes this quantity negative. 

Although the above choice of $\rho^\star$ is not Zariski-generic, this is immaterial.
Indeed, after choosing $\delta_2$ so that 
$C_0[\Phi,\rho_\epsilon](r_\star,r_\star)<-2\gamma$, an arbitrarily small perturbation of the coefficients $B_{p,u}$, within the real-valuedness, normalization, and chirality constraints and with cutoff $P=2L$, can be chosen Zariski-generic; since $C_0[\Phi,\rho](r_\star,r_\star)$ depends continuously on $B$, the perturbed density remains nonnegative and still satisfies $C_0[\Phi,\rho](r_\star,r_\star)<-\gamma<0$.
\end{remark}

To complement the theoretical stability analysis, we performed an additional numerical experiment to investigate the effect of non-uniform perturbations of the orientation distribution in the first dataset. Specifically, instead of sampling the primary poses from the uniform distribution, we used \(10^5\) poses sampled from the empirical pose distribution associated with the EMPIAR-10028 dataset used in \cite{zhong2021cryodrgn}, with pose files obtained from the accompanying public repository \cite{cryodrgn_empiar10028}. We then reran the experiments in the paper without modifying the reconstruction algorithm. A representative result is shown in Figure~\ref{fig:perturbed}. Overall, the reconstruction quality remains reasonably stable, suggesting that the proposed method is fairly robust to moderate deviations from uniformity in the orientation distribution of the primary dataset.

\begin{figure}[!ht]
    \centering
    \includegraphics[width=0.9\linewidth,
    trim={0cm 5cm 0cm 5cm},
    clip]{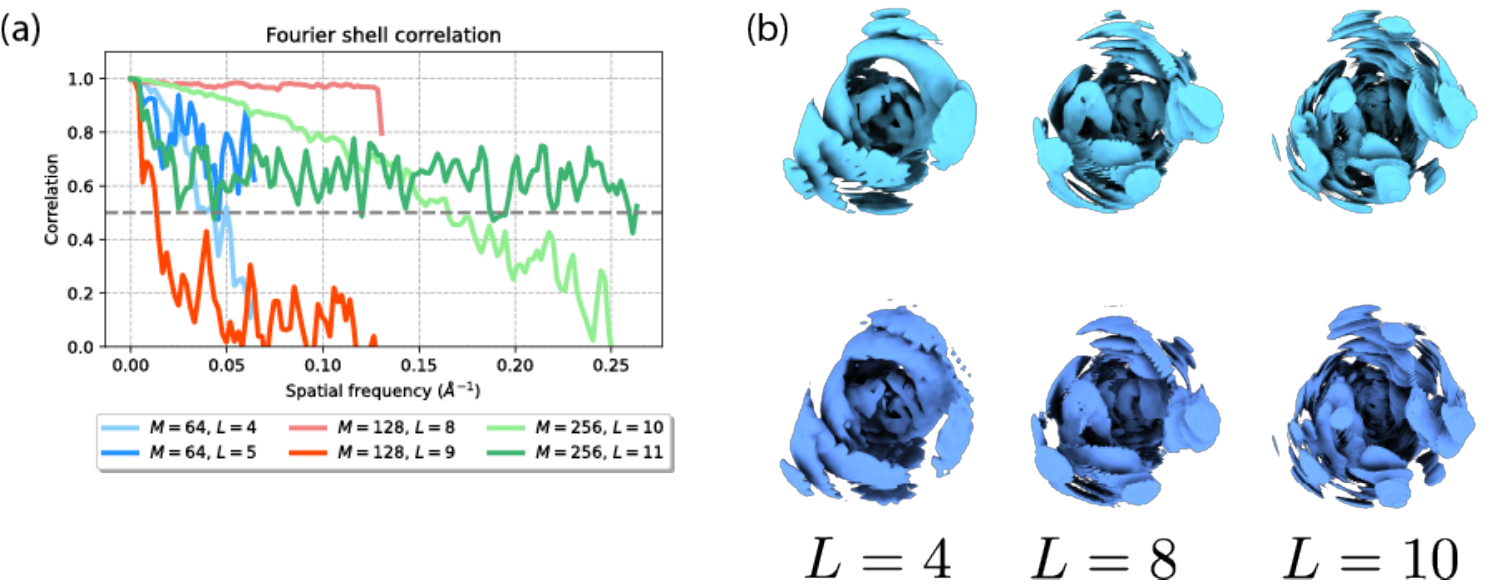}
    \caption{Result of the proposed algorithm when samples from the uniform distribution are replaced by \(10^5\) poses drawn from the empirical pose distribution associated with EMPIAR-10028. (a) Fourier-shell correlation between the reconstruction and the ground-truth structure. (b) Visualization of the reconstruction (top) and the ground-truth structure (bottom).}
    \label{fig:perturbed}
\end{figure}

\section{Conclusion}\label{sec:conclusion}

In this paper, we introduced a new moment-based algorithm for 3-D reconstruction in cryo-electron microscopy that fuses information from multiple datasets. 
By jointly using one dataset with uniformly distributed viewing angles and another with non-uniform orientations, we demonstrated that the 3-D reconstruction problem becomes more tractable when leveraging complementary orientation distributions. In particular, we established stronger uniqueness guarantees for recovery from first- and second-order moments than previously available. We further developed an efficient numerical algorithm based on alternating optimization, where each subproblem admits a closed-form update via convex relaxation. Numerical experiments confirm that our method achieves accurate molecular reconstructions up to prescribed bandlimits. Overall, this work highlights the potential of data fusion in cryo-EM, showing that combining datasets collected under different conditions can fundamentally improve identifiability and reconstruction accuracy.

Looking ahead, several promising research directions arise from this work. One important avenue is to develop statistical tests to assess whether experimentally collected datasets exhibit uniform distributions of viewing angles. Another is to extend our framework to handle higher bandlimits $L$ through suitable modifications of Algorithm~\ref{alg:modm}. Incorporating image shifts and optical aberration effects directly into the model, rather than relying on preprocessing, is another natural next step. In particular, random translational errors induce a frequency-dependent damping of the moments whose correction may become ill-conditioned, leading to additional computational challenges. Lastly but perhaps most importantly, the broader data fusion paradigm introduced here may offer substantial benefits within cryo-EM itself, particularly for maximum likelihood and Bayesian inference frameworks, and may also extend to other reconstruction algorithms and imaging modalities, offering a unified perspective on how complementary datasets can enhance structural inference.

\vspace{1em}

\paragraph{Acknowledgments.}  A.S. and S.X. (and O.M. for part of the work) were supported in part by AFOSR FA9550-23-1-0249,  NSF DMS 2510039, NIH/NIGMS R01GM13678001, and the Simons Foundation Math+X Investigator Award.  
J.K. was supported in part by NSF DMS 2309782, NSF DMS 2436499, NSF CISE-IIS 2312746, DE SC0025312, and the Sloan Foundation.  

\vspace{0.5em}

\paragraph{Data Availability.} 
The code implementing the proposed algorithm, the method of double moments, is available at 
\href{https://github.com/oscarmickelin/modm}{https://github.com/oscarmickelin/modm}. 

\vspace{0.5em}

\bibliographystyle{plain}

\bibliography{references}

\appendix
\section{Proofs for Section \ref{sec:setup}}
This appendix derives the basis expansion of the rotation distribution in Section~\ref{sec:basis_for_rotation_distribution} and the analytic moment expressions in Section~\ref{sec:moment_expressions}.

\subsection{Derivation of basis expansion of rotation distribution}\label{app:rotational_distribution}
We assume that the rotational distribution has density $\rho(R)$ with respect to the uniform distribution $\mathrm{d} R$. By Peter-Weyl theorem \cite[Thm. 8.13]{chirikjian2016harmonic}, any square integrable function $\rho(R)$ can be expanded as
\begin{align*}
\rho(R)=\sum_{p=0}^P \sum_{u,v=-p}^p B_{p,u,v} U_{uv}^p (R),
\end{align*}
as $P \rightarrow \infty$, where $R\in \mathsf{SO}(3)$ and $U_{uv}^p (R)$ is the $(u,v)$-th entry of the Wigner U-matrix $U^p(R)\in\C^{(2p+1)\times (2p+1)}$ \cite[Eq. 9.41]{chirikjian2016harmonic}. 

We next show how distributions invariant to in-plane rotations can be expressed as in \eqref{eq:def_rho_in_plane_uniform}. According to the definition \eqref{eq:def_in_plane_uniform}, we have
\begin{align*}
\sum_{p=0}^P \sum_{u,v=-p}^p \!\! B_{p,u,v} U_{uv}^p (R) &= \sum_{p=0}^P \sum_{u,v=-p}^p \!\! B_{p,u,v} U_{uv}^p (R z(\alpha)) = \sum_{p=0}^P \sum_{u,v=-p}^p \!\! B_{p,u,v} \Big( U^p(R) U^p(z(\alpha))   \Big)_{uv},
\end{align*}
where we use the homomorphism property of group representations \cite[Page 344]{chirikjian2016harmonic} for Wigner U-matrices in the last inequality and
\begin{align*}
U^p(z(\alpha))=\diag(e^{-\ii p\alpha}, e^{-\ii(p-1)\alpha},\ldots, e^{\ii p\alpha}).
\end{align*}
This gives
\begin{align*}
\sum_{p=0}^P \sum_{u,v=-p}^p B_{p,u,v} U_{uv}^p (R) = \sum_{p=0}^P \sum_{u,v=-p}^p B_{p,u,v} U_{uv}^p (R)e^{\ii v\alpha} ,
\end{align*}
which implies $B_{p,u,v}=0$ for $v\neq 0$. Thus, $\rho(R)$ admits the expansion
\begin{align*}
\rho(R)=\sum_{p=0}^P \sum_{u=-p}^p B_{p,u,0} U_{u0}^p(R).
\end{align*}
For simplicity, we drop the zero subscript and write
\begin{align}\label{eq:rho_expansion}
\rho(R) \mathrm{d}R =\sum_{p=0}^P \sum_{u=-p}^p B_{p,u} U_{u0}^p(R)\mathrm{d}R,
\end{align}
where $\mathrm{d}R$ denotes the uniform distribution and $U_{u0}^p(R)$ can explicitly be written as in \eqref{eq:explicit_U_expression}. 

To ensure that the density $\rho(R)$ is a real valued function, we impose the conjugate symmetry condition on its expansion coefficients. Specifically, we require
\begin{align*}
    \sum_{p=0}^P \sum_{u=-p}^p B_{p,u} U_{u0}^p(R)= \sum_{p=0}^P \sum_{u=-p}^p \overline{B_{p,u} U_{u0}^p(R)}.
\end{align*}
Using the identity $\overline{U_{u0}^p(R)}=(-1)^u U_{-u,0}^p(R)$ \cite[Page 68]{biedenharn1984angular}, we obtain the symmetry condition
\begin{align}\label{eq:B_symm}
    \overline{B_{p,u}}=(-1)^u B_{p,-u}.
\end{align}
Additionally, since $\rho(R)$ is a probability density over $\SO(3)$, it must integrate to one. Using the fact that $U_{00}^0(R)=1$ for any $R\in\mathsf{SO}(3)$, and applying the orthogonality of the matrix elements $\{U_{uv}^p(R):p\geq 0, -p\leq u,v\leq p\}$, we get
\begin{align*}
    B_{0,0}=\int_{\mathsf{SO}(3)}\sum_{p=0}^P \sum_{u=-p}^p B_{p,u} U_{u0}^p(R)\mathrm{d}R=\int_{\mathsf{SO}(3)} \rho(R) \mathrm{d}R =1.
\end{align*}

Finally, for chirality invariance, we deduce from the definition \eqref{eq:def_invariant_chirality} that
\begin{align*}
\sum_{p=0}^P \sum_{u=-p}^p B_{p,u} U_{u0}^p(R) = \sum_{p=0}^P \sum_{u=-p}^p B_{p,u} U_{u0}^p(JRJ).
\end{align*}
Using the identity
\begin{align*}
U_{u0}^p(JRJ)=(-1)^p U_{u0}^p(R),
\end{align*}
we obtain 
\begin{align*}
B_{p,u} = 0, \qquad \text{for all odd } p.
\end{align*}

\subsection{Derivation of moment expressions}\label{app:moment_expansions}
Using the Fourier slice theorem \cite{natterer2001mathematics} (c.f. its Cartesian version in \eqref{eq:Fourier_slice}) and the property of Wigner U-matrices \cite[Eqn. 9.49]{chirikjian2016harmonic}, we can write the Fourier transform of the projection images in \eqref{eq:def_proj_images} in spherical coordinates $(r,\theta,\varphi)$ as
\begin{align}\label{eq:expand_images}
\widehat{I}_R(r,\varphi)&= (R^\T\cdot \widehat\Phi) (r,\tfrac{\pi}{2},\varphi) = \sum_{\ell=0}^L \sum_{m=-\ell}^\ell  A_{\ell}^m (r) \sum_{n=-\ell}^\ell U_{mn}^{\ell} (R) Y_\ell^n(\tfrac{\pi}{2},\varphi)+\widehat{\varepsilon}(r,\varphi)\notag\\
&= \sum_{\ell=0}^L \sum_{m=-\ell}^\ell   \sum_{n=-\ell}^\ell A_{\ell}^m (r) U_{mn}^{\ell} (R)  e^{\ii n\varphi} N_\ell^n+\widehat{\varepsilon}(r,\varphi),
\end{align}
where $\widehat{\varepsilon}$ denotes the Fourier transform of the noise term, and in the last equality we use the explicit expressions for spherical harmonics \cite[Eqn. 4.36]{chirikjian2016harmonic} and define
\begin{align*}
N_\ell^n = \sqrt{\frac{2\ell +1}{4\pi}}\sqrt{\frac{(\ell-n)!}{(\ell+n)!}} P_\ell^n(0),
\end{align*}
where $P_l^n(x)$ denotes the associated Legendre polynomials. Notably, $N_\ell^n$ is always real-valued and, by the symmetry property of $P_l^n(x)$ \cite[Page 53]{chirikjian2016harmonic}, we have
\begin{align}\label{eq:N_symm}
    N_\ell^{-n} &= \sqrt{\frac{2\ell +1}{4\pi}}\sqrt{\frac{(\ell+n)!}{(\ell-n)!}} P_\ell^{-n}(0)\notag = \sqrt{\frac{2\ell +1}{4\pi}}\sqrt{\frac{(\ell+n)!}{(\ell-n)!}} \cdot (-1)^n \cdot \frac{(\ell-n)!}{(\ell+n)!}P_\ell^n(0)\notag\\
    &= (-1)^n N_\ell^n.
\end{align}

\subsubsection{First-order moment}\label{sec:1st_moment_derivation}
We first derive the formula for the first-order population moment. Using \eqref{eq:expand_images} and applying the zero mean property of Gaussian noise (see also \ref{eq:1st_moment_Cart}), we can write the first-order population moment as
\begin{equation}\label{eq:m1_expanded}
\begin{split}
m_1 (r,\varphi)&= \E_{\varepsilon}\Big[\int_{\mathsf{SO}(3)} \widehat{I_R}(r,\varphi) \rho(R)\mathrm{d}R\Big]\\
&=  \sum_{\ell=0}^L \sum_{m=-\ell}^\ell   \sum_{n=-\ell}^\ell A_{\ell}^m(r)   e^{\ii n\varphi} N_\ell^n \sum_{p=0}^P \sum_{u=-p}^p B_{p,u} \int_{\mathsf{SO}(3)}  U_{u0}^p(R)  U_{mn}^{\ell} (R)\mathrm{d}R.
\end{split}
\end{equation}
Note that Wigner U-matrices satisfy the orthogonality property \cite[Eq. 9.32]{chirikjian2016harmonic}
\begin{align*}
\int \overline{U_{mn}^\ell (R)} U_{uv}^p(R)\mathrm{d}R=\frac{1}{2\ell+1}\1_{\ell=p}\1_{u=m}\1_{v=n},
\end{align*}
and the symmetry property \cite[Page 68]{biedenharn1984angular}
\begin{align*}
\overline{U_{mn}^\ell (R)}=(-1)^{m+n} U_{-m,-n}^\ell (R) , 
\end{align*}
we can simplify \eqref{eq:m1_expanded} to
\begin{align}\label{eq:m1_expanded_2}
m_1 (r,\varphi)=\sum_{\ell=0}^{\min\{L,P\}} \sum_{m=-\ell}^\ell  (-1)^m\frac{1}{2\ell+1} N_\ell^0 \cdot A_{\ell }^m(r)  B_{\ell,-m},
\end{align}
where
\begin{align*}
N_\ell^0=\sqrt{\frac{2\ell+1}{4\pi}} P_\ell^0(0) 
\end{align*}
with
\begin{align*}
P_\ell^0(0) = \begin{cases}
    (-1)^{\ell/2}\cdot\frac{(\ell-1)!!}{\ell !!}, & \text{if}\ \ell \ \text{is even}, \\
    0, & \text{if}\ \ell \ \text{is odd}. 
\end{cases}
\end{align*}
The sum in \eqref{eq:m1_expanded_2} ranges over integer indices $(\ell,m)$ such that 
\begin{align*}
    0\leq \ell\leq \min\{L,P\},\quad
    \ell \ \text{is even}, \quad
    -\ell\leq m\leq\ell,
\end{align*}
which gives
\begin{align}\label{eq:1st_moment}
m_1 (r,\varphi)&=\sum_{\substack{\ell = 0,\\ \ell ~\text{even} }}^{ \min\{L,P\}} \sum_{m=-\ell}^\ell  (-1)^m\frac{N_\ell^0  A_{\ell}^m(r)}{2\ell+1}   B_{\ell,-m} =\sum_{\substack{\ell = 0,\\ \ell ~\text{even} }}^{ \min\{L,P\}}  \sum_{m=-\ell}^\ell  \frac{N_\ell^0  A_{\ell}^m(r)}{2\ell+1}   \overline{B_{\ell,m}}.
\end{align}
Thus, we can write the first-order moment in the following matrix form:
\begin{align}\label{eq:1st_moment_matrix}
m_1(r,\varphi)&=\sum_{\ell = 0}^{ \min\{L,P\}} \sum_{m=-\ell}^\ell  \frac{N_\ell^0 \1_{\{\ell \text{ is even}\}} }{2\ell+1}  A_{\ell m}(r)  \overline{B_{\ell,m}}=\sum_{0\leq \ell\leq \min\{L,P\}} A_\ell (r) (\mathfrak{B}_\ell)^\mathsf{H}  ,
\end{align}
where the terms
\begin{align*}
A_\ell(r) \in \C^{1\times (2\ell+1)}~~~{\rm and}~~~\mathfrak{B}_\ell\in\C^{1\times (2\ell+1)}
\end{align*}
for $0\leq \ell\leq \min\{L,P\}$ are defined by
\begin{align}
(A_{\ell}(r))_m&=A_{\ell}^m(r),\label{def:A_lr}\\
\mathcal{N}_\ell^0&= N^0_{\ell}\cdot \1_{\{\ell \equiv 0 ~(\text{mod 2})\}}\cdot \1_{\{\ell\geq 0\}},\label{def:cN_0}\\
(\mathfrak{B}_\ell)_m&= \frac{\mathcal{N}^0_{\ell} B_{\ell,m}}{2\ell+1} ,\label{def:frakB_l}
\end{align}
for $-\ell\leq m \leq \ell$. Here $A_\ell(r)$ and $\mathcal{N}_\ell^0$ are the same as defined in \eqref{eq:def_Al} and \eqref{eq:def_calN}, respectively.

\subsubsection{Second-order moment}\label{sec:2nd_moment_derivation}
We now derive the formula for the second-order population moment and prove Proposition \ref{prop:2nd_moment}. Using \eqref{eq:expand_images} and noting that the noise term $\widehat{\varepsilon}$ is independent from the signal term, we can write the second order population moment from \eqref{eq:2nd_moment_Cart} as
\begin{align*}
&m_2(r,\varphi,r',\varphi')\\
&= \int_{\mathsf{SO}(3)} \Big( \sum_{\ell=0}^L \sum_{m=-\ell}^\ell   \sum_{n=-\ell}^\ell A_{\ell}^m(r) U_{mn}^{\ell} (R)  e^{\ii n\varphi} N_\ell^n\Big)  \Big( \sum_{\ell'=0}^L \sum_{m'=-\ell'}^{\ell'}   \sum_{n'=-\ell'}^{\ell'} \overline{A_{\ell'}^{m'}(r')} \overline{U_{m'n'}^{\ell'} (R)}  e^{-\ii n'\varphi'} N_{\ell'}^{n'}\Big)\\
&\quad\quad \times \Big(\sum_{p=0}^P \sum_{u=-p}^p B_{p,u} U_{u0}^p(R)\Big)\mathrm{d}R.
\end{align*}
The product of two Wigner matrix entries can be expressed as a linear combination of Wigner matrix entries \cite[Eqn. 9.64]{chirikjian2016harmonic}, 
\begin{align}\label{eq:prod_Wigner}
U_{mn}^{\ell}(R) U_{m'n'}^{\ell'} (R) = \sum_{\ell''=|\ell-\ell'|}^{\ell+\ell'} \mathcal{C}_{\ell''} (\ell,\ell',m,m',n,n')U_{m+m',n+n'}^{\ell''} (R),
\end{align} 
where 
\begin{align}\label{eq:myCG}
\mathcal{C}_{\ell''} (\ell,\ell',m,m',n,n')=C(\ell,m;\ell',m'|\ell'',m+m') C(\ell,n;\ell',n'|\ell'',n+n'),
\end{align}
is the product of two Clebsch-Gordan coefficients. This product is nonzero only if $(\ell,\ell',\ell'')$ satisfy the triangle inequalities $|\ell-\ell'|\leq \ell''\leq \ell+\ell'$. Note that we also use the fact that the Clebsch-Gordan coefficient $C(\ell,m;\ell',m'|\ell'', m'')$ is nonzero only if $m''=m+m'$ to simplify the formula in \eqref{eq:prod_Wigner} \cite[Eqn 2.41]{bohm2013quantum}. Using the identity $\overline{U^a_{b,c}(R)} = U^{a}_{-b,-c}(R)$, we obtain
\begin{align*}
&\int_{\mathsf{SO}(3)} U_{mn}^\ell(R) \overline{U_{m'n'}^{\ell'}(R)} U_{u0}^p(R) \mathrm{d}R\\
&=\int_{\mathsf{SO}(3)} (-1)^{m'+n'} U_{mn}^\ell(R)  U_{-m',-n'}^{\ell'}(R) U_{u0}^p(R) \mathrm{d}R\\
&=\int_{\mathsf{SO}(3)} (-1)^{m'+n'}\sum_{\ell''=|\ell-\ell'|}^{\ell+\ell'} \mathcal{C}_{\ell''} (\ell,\ell',m,-m',n,-n') U_{m-m',n-n'}^{\ell''} (R) U_{u0}^p(R) \mathrm{d}R\\
&=\int_{\mathsf{SO}(3)} (-1)^{m+n}\sum_{\ell''=|\ell-\ell'|}^{\ell+\ell'} \mathcal{C}_{\ell''} (\ell,\ell',m,-m',n,-n')  \overline{U_{-m+m',-n+n'}^{\ell''} (R)} U_{u0}^p(R) \mathrm{d}R\\
&= (-1)^{m+n}\sum_{\ell''=|\ell-\ell'|}^{\ell+\ell'} \mathcal{C}_{\ell''} (\ell,\ell',m,-m',n,-n')  \frac{1}{2\ell''+1} \1_{\ell''=p}\1_{m'-m=u}\1_{n=n'}.
\end{align*}
It follows that
\begin{align*}
&m_2(r,\varphi,r',\varphi')\\
&= \sum_{\ell=0}^L  \sum_{\ell'=0}^L \sum_{m=-\ell}^{\ell}\sum_{m'=-\ell'}^{\ell'}\sum_{n=-\ell}^{\ell} \sum_{n'=-\ell'}^{\ell'} e^{\ii n\varphi} e^{-\ii n'\varphi'} N_\ell^n N_{\ell'}^{n'} A_{\ell}^m(r) \overline{A_{\ell'}^{m'}(r')} \\
&  \times \sum_{p=0}^P\sum_{u=-p}^p \sum_{\ell''=|\ell-\ell'|}^{\ell+\ell'} B_{p,u} (-1)^{m+n}  \frac{\mathcal{C}_{\ell''} (\ell,\ell',m,-m',n,-n')}{2\ell''+1} \1_{p=\ell''}\1_{u=m'-m}\1_{n=n'}.
\end{align*}
Rearranging the order of summations gives
\begin{align*}
&m_2(r,\varphi,r',\varphi')\\
&=\sum_{n=-L}^{L} \sum_{n'=-L}^{L} \sum_{\ell=|n|}^L  \sum_{\ell'=|n'|}^L \sum_{m=-\ell}^{\ell}\sum_{m'=-\ell'}^{\ell'} e^{\ii n\varphi} e^{-\ii n'\varphi'} N_\ell^n N_{\ell'}^{n'} A_{\ell}^m(r) \overline{A_{\ell'}^{m'}(r')} \\
&  \times \sum_{p=0}^P\sum_{u=-p}^p \sum_{\ell''=|\ell-\ell'|}^{\ell+\ell'} B_{p,u} (-1)^{m+n}  \frac{\mathcal{C}_{\ell''} (\ell,\ell',m,-m',n,-n')}{2\ell''+1} \1_{p=\ell''}\1_{u=m'-m}\1_{n=n'}.
\end{align*}
This further gives
\begin{align*}
&m_2(r,\varphi,r',\varphi')\\
&= \sum_{n=-L}^{L} \sum_{\ell=|n|}^L  \sum_{\ell'=|n|}^L e^{\ii n(\varphi-\varphi')} N_\ell^n N_{\ell'}^{n}\\
&\times \sum_{m=-\ell}^{\ell}\sum_{m'=-\ell'}^{\ell'} \sum_{p=0}^P\sum_{u=-p}^p \sum_{\ell''=|\ell-\ell'|}^{\ell+\ell'} A_{\ell}^m(r) \overline{A_{\ell'}^{m'}(r')}  B_{p,u} (-1)^{m+n} \mathcal{C}_{\ell''} (\ell,\ell',m,-m',n,-n) \frac{\1_{p=\ell''}\1_{u=m'-m}}{2\ell''+1} \\
&= \sum_{n=-L}^{L} \sum_{\ell=|n|}^L  \sum_{\ell'=|n|}^L e^{\ii n(\varphi-\varphi')} N_\ell^n N_{\ell'}^{n} \\
& \quad \quad  \times \sum_{\ell''=|\ell-\ell'|}^{\min\{\ell+\ell',P\}} \sum_{\substack{-\ell\leq m\leq \ell \\ -\ell'\leq m'\leq \ell' \\ |m'-m|\leq \ell''}}  (-1)^{m+n} \mathcal{C}_{\ell''} (\ell,\ell',m,-m',n,-n) \frac{A_{\ell}^m(r) \overline{A_{\ell'}^{m'}(r')}}{2\ell''+1} B_{\ell'',m'-m}.
\end{align*}
According to \cite[Lemma D.1]{fan2024maximum}, we get $N_{\ell}^n=0$ for odd $\ell+n$, which further implies
\begin{align}\label{eq:2nd_moment}
&m_2(r,\varphi,r',\varphi')\notag\\
&= \sum_{n=-L}^{L} e^{\ii n(\varphi-\varphi')} \sum_{\substack{|n|\leq \ell\leq L \\ \ell\equiv n ~ (\text{mod} \ 2)}}  \sum_{\substack{|n|\leq \ell'\leq L \\ \ell'\equiv n ~ (\text{mod} \ 2)}}  N_\ell^n N_{\ell'}^n \notag\\
& \quad \quad  \sum_{\ell''=|\ell-\ell'|}^{\min\{\ell+\ell',P\}} \sum_{\substack{-\ell\leq m\leq \ell \\ -\ell'\leq m'\leq \ell' \\ |m'-m|\leq \ell''}}  (-1)^{m+n}\mathcal{C}_{\ell''} (\ell,\ell',m,-m',n,-n) \frac{A_{\ell}^m(r) \overline{A_{\ell'}^{m'}(r')}}{2\ell''+1} B_{\ell'',m'-m}.
\end{align}
As anticipated by the fact that the rotational distribution $\rho$ is in-plane uniform, the second-order moment depends only on $\varphi-\varphi'$ as shown above. Rearranging the sum and using the notation $\mathcal{N}_{\ell}^n$ as defined in \eqref{eq:def_calN} gives
\begin{align*}
&m_2(r,\varphi,r',\varphi')\notag\\
&= \sum_{n=-L}^{L} e^{\ii n(\varphi-\varphi')} \sum_{\ell=0}^L\sum_{\ell'=0}^L  \mathcal{N}_\ell^n \mathcal{N}_{\ell'}^n \notag\\
&   \sum_{\substack{-\ell\leq m\leq \ell \\ -\ell'\leq m'\leq \ell' }}  
\sum_{\ell''=\max\{|m-m'|, |\ell-\ell'|\}}^{\min\{\ell+\ell',P\}} 
(-1)^{m+n} \frac{\mathcal{C}_{\ell''} (\ell,\ell',m,-m',n,-n)}{2\ell''+1}A_{\ell}^m(r) B_{\ell'',m'-m} \overline{A_{\ell'}^{m'}(r')}.
\end{align*}
Thus, we can write the second-order moment more concisely in the following matrix form:
\begin{align}\label{eq:2nd_moment_matrix}
&m_2(r,\varphi,r',\varphi') = \sum_{n=-L}^{L} e^{\ii n(\varphi-\varphi')} \sum_{0 \leq \ell\leq L }  \sum_{0 \leq \ell' \leq L }   A_{\ell}(r) \mathcal{B}^n_{\ell,\ell'} (A_{\ell'}(r'))^{\mathsf{H}},
\end{align}
where the terms
\begin{align*}
    A_{\ell}(r)\in\C^{1\times (2\ell+1)}, \quad \mathcal{B}^n_{\ell,\ell'}\in\C^{(2\ell+1)\times (2\ell'+1)},
\end{align*}
for $ \ell, \ell' \in \{0,\ldots , L\}$, $n\in\{-L,\ldots,L\}$ are defined as \eqref{eq:def_Al} and \eqref{eq:def_calB}. 

In addition, we show the Hermitian property of $(\mathcal{B}^n_{\ell',\ell})_{m', m}$ for potential application in the paper. By the symmetry property of Clebsh-Gordan coefficients \cite[Eqn. 2.47]{bohm2013quantum},
\begin{equation}\label{eq:cC_symm}
    \begin{split}
    \mathcal{C}_{\ell''} (\ell',\ell,m',-m,n,-n)&= C(\ell',m';\ell,-m|\ell'',-m+m') \cdot C(\ell',n;\ell,-n|\ell'',0)\\
    &= C(\ell,-m;\ell',m'|\ell'',-m+m') \cdot C(\ell,-n;\ell',n|\ell'',0)\\
    &= C(\ell,m;\ell',-m'|\ell'',m-m') \cdot C(\ell,n;\ell',-n|\ell'',0)\\
    &= \mathcal{C}_{\ell''} (\ell,\ell',m,-m',n,-n).
    \end{split}
\end{equation}
Applying \eqref{eq:B_symm} and \eqref{eq:N_symm}, we obtain
\begin{align*}
\overline{(\mathcal{B}^n_{\ell',\ell})_{m', m}} = &\sum_{\ell''=\max\{|m-m'|, |\ell-\ell'|\}}^{\min\{\ell+\ell',P\}}  (-1)^{m'+n}
\mathcal{N}^n_{\ell'} \mathcal{N}^{n}_{\ell} \frac{\mathcal{C}_{\ell''} (\ell',\ell,m',-m,n,-n)}{2\ell''+1} \overline{B_{\ell'',m-m'}}\notag\\
&= \sum_{\ell''=\max\{|m-m'|, |\ell-\ell'|\}}^{\min\{\ell+\ell',P\}} (-1)^{m'+n}
\mathcal{N}^n_{\ell'} \mathcal{N}^{n}_{\ell} \frac{\mathcal{C}_{\ell''} (\ell',\ell,m',-m,n,-n)}{2\ell''+1} (-1)^{m-m'}B_{\ell'',-m+m'}\notag\\
&= \sum_{\ell''=\max\{|m-m'|, |\ell-\ell'|\}}^{\min\{\ell+\ell',P\}} (-1)^{m+n}
\mathcal{N}^n_{\ell} \mathcal{N}^{n}_{\ell'} \frac{\mathcal{C}_{\ell''} (\ell,\ell',m,-m',n,-n)}{2\ell''+1} B_{\ell'',-m+m'}\notag\\
&=(\mathcal{B}^n_{\ell,\ell'})_{m,m'}.
\end{align*}
Hence, we get
\begin{align*}
    \mathcal{B}^n_{\ell,\ell'} = (\mathcal{B}^n_{\ell',\ell})^{\mathsf{H}}.
\end{align*}
Also, according to the definition \eqref{eq:def_calN}, we know $\mathcal{N}_\ell^n\mathcal{N}_{\ell'}^n =0$ whenever $\ell\not\equiv \ell'~(\mathrm{mod}~2)$, which implies for any $-L\leq n\leq L$,
\begin{align*}
    \mathcal{B}^n_{\ell,\ell'}=0,
\end{align*}
whenever $\ell\not\equiv \ell'~(\mathrm{mod}~2)$.

\section{Proofs for Section~\ref{sec:uniqueness}}\label{sec:appendixB}
This appendix provides proofs of auxiliary results needed for the uniqueness theorem. 

\subsection{Proof of Lemma~\ref{lem:symmetry}}
In fact, we claim that the distributions of tomographic projection images $I_R$ with noise removed (that is, \eqref{eq:def_proj_images} without the $\varepsilon(x,y)$ term) as generated by $(\Phi, \rho_1)$ and $(\Phi, \rho_2)$ match the distributions generated by $(\tilde{\Phi}, {\rho_1})$ and $(\tilde{\Phi},\tilde{\rho}_2)$, respectively.  
Consequently, 
$m_k[\Phi, \rho_1] = m_k[\tilde{\Phi}, {\rho_1}]$ and 
$m_k[\Phi, \rho_2] = m_k[\tilde{\Phi}, \tilde{\rho}_2]$ 
for all $k \geq 1$.  In particular,  \eqref{eq:lemma-want} holds as wanted.

Let us justify this claim.  
First compare $(\Phi, \rho_2)$ and $(\tilde{\Phi}, \tilde{\rho}_2)$.
Under $(\tilde{\Phi}, \tilde{\rho}_2)$, the distribution of noiseless tomographic projection images draws with density $\tilde{\rho}_2(R) = \rho_2(J^{\epsilon}SRJ^{\epsilon})$,
\begin{equation}\label{eq:lemma-proof}
\int_{-\infty}^{\infty} (R^{\top} \cdot \tilde{\Phi})(\x) \der x_3 \, = \, \int_{-\infty}^{\infty} {\Phi}(J^{\epsilon} S R\x) \der x_3 \, = \, \int_{-\infty}^{\infty} \Phi(J^{\epsilon} S R J^{\epsilon} J^{\epsilon} \x) \der x_3 \, = \,  \int_{-\infty}^{\infty} \Phi(J^{\epsilon} S R J^{\epsilon}\x) \der x_3.
\end{equation}
Here the last equality in \eqref{eq:lemma-proof} is by a change of variable replacing $x_3$ by $(-1)^{\epsilon}x_3$.
Meanwhile, under $(\Phi, \rho_2)$ the distribution of noiseless tomographic projection images draws 
\begin{equation}
\int_{-\infty}^{\infty} \Phi(J^{\epsilon} S R J^{\epsilon}\x) \der x_3,
\end{equation}
with density $\rho_2(J^{\epsilon}SRJ^{\epsilon})$.  The distributions are identical, as claimed. 
Comparing $(\Phi, \rho_1)$ and $(\tilde{\Phi}, {\rho_1})$, the image distributions likewise are the same, because ${\rho_1}(R) = \rho_1(J^{\epsilon}SRJ^{\epsilon})$ since  $\rho_1$ is the uniform distribution.  This justifies our claim and finishes the proof of the lemma. \qedhere

\subsection{Proof of Lemma~\ref{lem:technical1}}
In fact, we will show that the affine-linear map 
\begin{equation} \label{eq:map-nunderline}
B_{2L} \mapsto \mathcal{B}^{\underline{n}}_{L,L}
\end{equation}
is injective, where $\underline{n}$ is the element of $\{0, 1\}$ satisfying $\underline{n} \equiv L \,\, (\textup{mod } 2)$. 
By \eqref{eq:def_calB}, the linear part of \eqref{eq:map-nunderline} (i.e., dropping an additive constant depending only $B_{p}$ with $p \leq 2L-2$) reads:
\begin{equation*}
B_{2L} \,\,  \mapsto \,\, (\mathcal{B}^{\underline{n}}_{L,L})_{m,m'}= (-1)^{m+\underline{n}}\mathcal{N}^{\underline{n}}_{L} \mathcal{N}^{\underline{n}}_{L}  \frac{\mathcal{C}_{2L} (L,L,m,-m',\underline{n},-\underline{n})} {4L+1}B_{2L,m'-m}.
\end{equation*}
By Lemmas~\ref{lem:nonvanish-CG} and \ref{lem:nonvanish-N}, the coefficient of $B_{2L, m'-m}$ is nonzero for all $-L \leq m,m' \leq L$.  Clearly, \eqref{eq:map-nunderline} is injective and so is \eqref{eq:allBnLL}. \qedhere

\subsection{Proof of Lemma~\ref{lem:technical2}}
Since $\left(\mathcal{B}^{n}_{L, \ell'}\right)$ depend polynomially on $\left(B_p\right)$ and full column rank of \eqref{eq:big-matrix} is a Zariski-open condition \cite{sommese2005numerical}, it is enough to exhibit a single instance of $B_p$ such that \eqref{eq:big-matrix} has full column rank.  Furthermore, in producing such an instance we can temporarily drop the conjugate symmetry condition in \eqref{eq:constraints_on_B} on $\left(B_p\right)$, because this condition defines a dense subset with respect to the complex Zariski topology.  Moreover, we can also temporarily drop the normalization condition in \eqref{eq:constraints_on_B}, because $\left(\mathcal{B}^{n}_{L, \ell'}\right)$ depend homogeneously on $\left(B_p\right)$.

Now, let us produce a desired instance.  If $L=4$ or $L=5$,  randomly generate $(B_p)$ and verify that \eqref{eq:big-matrix} has rank $2L+1$ (to be fully rigorous, in exact arithmetic in a finite extension of the rational numbers $\mathbb{Q}$ using the formulas for $\mathcal{C}_{\ell''}(\ell, \ell', m, m',n,n')$ and  $\mathcal{N}^n_{\ell}$ in the proofs of Lemmas~\ref{lem:nonvanish-CG} and \ref{lem:nonvanish-N}).  
Meanwhile, for $L \geq 6$ choose $\left(B_p\right)$ as follows:
\begin{align*}
B_{p,u} = \begin{cases}  \textup{nonzero and sufficiently large in magnitude if } (p,u) = (2L-2, 2) \\
\textup{nonzero and sufficiently small in magnitude if }  (p,u) = (2L-4, -2L+7) \\
0 \textup{ for all other } (p, u).\\ \end{cases}
\end{align*}
In \eqref{eq:big-matrix}, consider the submatrix $[\mathcal{B}^{\underline{n}}_{L, L-2} | \mathcal{B}^{\underline{n}}_{L,L-4}] \in \mathbb{C}^{(2L+1) \times (4L-10)}$ where $\underline{n} \in \{0,1\}$ such that $\underline{n} \equiv L \,\, (\textup{mod } 2)$.  The rows of the submatrix are indexed by $m \in \{-L, -L +1, \ldots, L\}$.  The columns of its first and second block are indexed by $m' \in \{-(L-2), -(L-2)+1, \ldots, L-2\}$ and $m' \in \{-(L-4), -(L-4)+1, \ldots, L-4\}$ respectively.  With the above choice of $(B_p)$, the block $\mathcal{B}^{\underline{n}}_{L, L-2}$ has support contained in the main diagonal where $m'-m = -(L-2) - (-L) = 2$ and the off-diagonal where $m'-m = -(L-2)-(L-5) = -2L+7$ by \eqref{eq:def_calB}.  Moreover, the entries in the main diagonal are nonzero by Lemmas~\ref{lem:nonvanish-CG} and \ref{lem:nonvanish-N}, and larger in magnitude than the entries in the off-diagonal.  Meanwhile, $\mathcal{B}^{\underline{n}}_{L,L-4}$ is supported on its off-diagonal where $m'-m=-(L-4)-(L-3)=-2L+7$, and these entries are  nonzero by Lemmas~\ref{lem:nonvanish-CG} and \ref{lem:nonvanish-N}.  
It follows that the leftmost $(2L+1) \times (2L+1)$ submatrix of $[\mathcal{B}^{\underline{n}}_{L, L-2} | \mathcal{B}^{\underline{n}}_{L,L-4}]$ is columnwise diagonally dominant, and therefore of full rank.    This finishes the case $L \geq 6$.

\subsection{Nonvanishing of constants}

\begin{lemma} \label{lem:nonvanish-CG}
Let integers $\ell, \ell', m, m', n, n'$ satisfy $\ell \geq |m|,|n|$ and $\ell' \geq |m'|,|n'|$.  Then  the constant $\mathcal{C}_{\ell + \ell'}(\ell, \ell', m, m', n, n')$  \eqref{eq:myCG0} is nonzero.
\end{lemma}
\begin{proof}
For all septuples of integers satisfying 
\begin{align*}
\ell, \ell', \ell'' \geq 0, \quad  |\ell - \ell'| \leq \ell'' \leq \ell + \ell', \quad \ell \geq |m|, |n| \quad \text{and} \quad \ell' \geq |m'|, |n'|,
\end{align*}
by \eqref{eq:myCG0} 
$\mathcal{C}_{\ell''} (\ell,\ell',m,m',n,n') = C(\ell,m;\ell',m'|\ell'',m+m') C(\ell,n;\ell',n'|\ell'',n+n')$; by \cite{bohm2013quantum}  
\begin{align*}
&C(\ell,m;\ell',m'|\ell'',m+m') \\[4pt] =& \,\, \sqrt{\frac{(2\ell'' + 1)(\ell + \ell' - \ell'')!(\ell + \ell'' - \ell')!(\ell' + \ell'' - \ell)!}{(\ell+\ell'+\ell''+1)!}} \\[2pt]
& \times \sqrt{(\ell-m)!(\ell+m)!(\ell'-m')!(\ell'+m')!(\ell''-m'')!(\ell''+m'')!} \\[2pt]
& \times \sum_k \frac{(-1)^k}{k!(\ell+\ell'-\ell''-k)!(\ell-m-k)!(\ell'+m'-k)!(\ell''-\ell'+m+k)!(\ell''-\ell-m'+k)!}
\end{align*}
where the summation is over all integers $k$ such that the argument of every factorial is nonnegative; and \cite{bohm2013quantum} the analogous formula for  $C(\ell, n; \ell', n' | \ell'', n+ n')$ holds.
In the special case $\ell'' = \ell + \ell'$, the summations over $k$ collapse to one term where $k=0$.  It follows $C(\ell,m;\ell',m'|\ell + \ell',m+m')  \neq 0$ and $C(\ell,n;\ell',n'|\ell + \ell',n+n') \neq 0$, hence $\mathcal{C}_{\ell + \ell'}(\ell, \ell', m, m', n, n') \neq 0$.  
\end{proof}

\begin{lemma} \label{lem:nonvanish-N}
Let integers $\ell$ and $n$ satisfy $\ell \geq |n|$ and $\ell \equiv n \,\, (\textup{mod } 2)$. Then the constant $\mathcal{N}^n_{\ell}$  \eqref{eq:def_calN} is nonzero.
\end{lemma}
\begin{proof}
Under the stated conditions on $\ell$ and $n$, by \eqref{eq:def_calN}  
\begin{equation*}
    \mathcal{N}^n_{\ell} = \sqrt{\frac{2\ell+1}{4\pi}} \sqrt{\frac{(\ell-n)!}{(\ell+n)!}} P^{n}_{\ell}(0),
\end{equation*}  
where $P^{n}_{\ell}(x)$ denotes an associated Legendre polynomial.
By \cite{chirikjian2016harmonic}, for $-1 \leq x \leq 1$ it holds
\begin{equation*}
P^{n}_{\ell}(x) = \frac{(-1)^n}{2^{\ell} \ell!} (1-x^2)^{n/2} \frac{\der ^{\ell+n}}{\der x^{\ell+n}}(x^2-1)^{\ell}.
\end{equation*}
The coefficient of $x^{\ell+n}$ in $(x^2-1)^{\ell}$ is $\binom{\ell}{(\ell+n)/2} (-1)^{(\ell -n)/2}$, therefore $\frac{\der^{\ell+n}}{\der x^{\ell+n}}(x^2-1)^{\ell}|_{x=0} \, = (\ell+n)! \binom{\ell}{(\ell+n)/2} (-1)^{(\ell -n)/2}$.  Hence $P^{n}_{\ell}(0) = \frac{(-1)^{(\ell+n)/2}}{2^{\ell} \ell!} (\ell+n)! \binom{\ell}{(\ell+n)/2}$ and
\begin{equation}
\mathcal{N}^{n}_{\ell} \, = \, \frac{(-1)^{(\ell+n)/2}}{2^{\ell}} \, \sqrt{\frac{2\ell+1}{4\pi}} \, 
\frac{\sqrt{(\ell+n)!(\ell-n)!}}{(\frac{\ell+n}{2})! (\frac{\ell-n}{2})!}. 
\end{equation}
In particular, $\mathcal{N}^{n}_{\ell}$ is nonzero as claimed.
\end{proof}

\section{Proofs for Section~\ref{sec:stability}}\label{sec:appendixC}

\subsection{Proof of Proposition~\ref{prop:Kam_mat_stability}}\label{sec:Kam_mat_stability_proof}
\begin{proof}
Recall from Proposition~\ref{prop:2nd_moment} that
\begin{align*}
m_2[\Phi,\rho](r,r',\psi) &= \sum_{n=-L}^{L} e^{\ii n\psi} \sum_{a=0}^L  \sum_{b=0}^L   A_{a}(r) \mathcal{B}^n_{a,b} (A_{b}(r'))^{\mathsf{H}}.
\end{align*}
Plugging this into \eqref{eq:def_Kam_general} yields
\begin{align}
C_\ell[\Phi,\rho](r,r')=& 2\pi(2\ell+1)\!\cdot\!\int_0^{\pi} \sum_{n=-L}^{L} e^{\ii n\psi} \sum_{a=0}^L  \sum_{b=0}^L   A_{a}(r) \mathcal{B}^n_{a,b} (A_{b}(r'))^{\mathsf{H}}  P_\ell(\cos \psi) \sin (\psi )d\psi\notag\\
=& \sum_{n=-L}^{L} \alpha_{\ell}^n \sum_{a=0}^L  \sum_{b=0}^L   A_{a}(r) \mathcal{B}^n_{a,b} (A_{b}(r'))^{\mathsf{H}},\label{eq:Cl_1}  
\end{align}
where $\alpha_\ell^n$ is defined as \eqref{eq:def_alpha}. Recall from \eqref{eq:def_calB} that, for  $a,b \in \{0,\ldots , L\}$, $n \in \{-L, \ldots, L\}$, and $-\ell\leq m,m'\leq \ell$,
\begin{align*}
(\mathcal{B}^n_{a,b})_{m,m'}&= (-1)^{m+n}\mathcal{N}^n_{a} \mathcal{N}^n_{b} \sum_{\ell''=\max\{|m-m'|, |a-b|\}}^{\min\{a+b,P\}} 
\!\!\!\!\!\! \frac{\mathcal{C}_{\ell''} (a,b,m,-m',n,-n)} {2\ell''+1}B_{\ell'',m'-m}.
\end{align*}
Renaming $p=\ell''$ and $u=m'-m$ gives
\begin{align*}
\mathcal{B}^n_{a,b}=\sum_{p=0}^{\min\{2L,P\}}\sum_{u=-p}^p B_{p,u} E_{a,b;\, p,u}^{n},
\end{align*}
where $E_{a,b;\, p,u}^{n}$ is defined as \eqref{eq:def_E}. Plugging this into \eqref{eq:Cl_1} yields
\begin{align*}
C_\ell[\Phi,\rho](r,r')=&\sum_{p=0}^{\min\{2L,P\}}\sum_{u=-p}^p B_{p,u} \Big(\sum_{n=-L}^{L} \alpha_{\ell}^n \sum_{a=0}^L  \sum_{b=0}^L   A_{a}(r) E_{a,b;\, p,u}^{n} (A_{b}(r'))^{\mathsf{H}}\Big)\\
=&\sum_{p=0}^{\min\{2L,P\}}\sum_{u=-p}^p B_{p,u} H_{\ell;p,u}[\Phi](r,r')\\
=&\sum_{p=0}^{2L}\sum_{u=-p}^p B_{p,u} H_{\ell;p,u}[\Phi](r,r'),
\end{align*}
where $H_{\ell;p,u}[\Phi](r,r')$ is defined as \eqref{eq:def_Kam_response} and the last equality follows from $P\geq 2L$.

By definition,
\begin{align*}
\beta_{2L}(\rho)= \|\rho-\rho_1^*\|_{L_2(\SO(3)}=&\Big(\int_{\SO(3)}\big| \sum_{p=1}^{P} \sum_{u=-p}^p B_{p,u} U_{u0}^p(R) \big|^2\, \der R\Big)^{1/2}\\
=&\Big(\sum_{p=1}^{2L}\sum_{u=-p}^p \frac{|B_{p,u}|^2}{2p+1} \Big)^{1/2},
\end{align*}
where the last equality follows from the orthogonality property of Wigner U-matrices. Note that
\begin{align*}
H_{\ell;0,0}[\Phi](r,r')=C_\ell[\Phi,\rho_1^*](r,r'),
\end{align*}
which implies
\begin{align*}
\Delta_\ell(\rho)=\sum_{p=1}^{2L}\sum_{u=-p}^p B_{p,u} H_{\ell;p,u}[\Phi].
\end{align*}
Hence, the Cauchy-Schwarz inequality gives \eqref{eq:perturbation_Kam}.

\end{proof}

\subsection{Proof of Theorem~\ref{thm:stability}}\label{sec:stability_proof}

We will use the following standard factor perturbation lemma, which is the complex-valued analogue of \cite[Lemma~5.4]{tu2016low}. For completeness, we provide the proof.

\begin{lemma}[Factor perturbation]
Let $X,Y\in\mathbb C^{m\times q}$, and assume that $X$ has full column rank with
$s:=\sigma_{\min}(X)>0$. Define
\[
\operatorname{dist}_F(Y,X):=\min_{U\in\mathsf U(q)}\|Y-XU\|_F.
\]
Then
\[
\operatorname{dist}_F(Y,X)
\leq
c_0\,\frac{\|YY^{\mathsf{H}}-XX^{\mathsf{H}}\|_F}{s},
\qquad
c_0:=\frac{1}{\sqrt{2(\sqrt2-1)}}.
\]
\end{lemma}

\begin{proof}[Proof of the lemma]
Let $U_*$ solve the Procrustes problem
\[
U_*\in\arg\min_{U\in\mathsf U(q)}\|Y-XU\|_F.
\]
Replacing $X$ by $XU_*$ does not change $XX^{\mathsf{H}}$ or $\sigma_{\min}(X)$, so we may
assume that $U_*=I$.  Set
\[
H:=Y-X.
\]
The Procrustes optimality condition implies $Y^{\mathsf{H}} X$ is a positive semidefinite Hermitian matrix. Hence
\[
H^{\mathsf{H}} X=Y^{\mathsf{H}} X-X^{\mathsf{H}} X
\]
is Hermitian. Write
\[
B:=H^{\mathsf{H}} H,
\qquad
D:=X^{\mathsf{H}} X,
\qquad
S:=H^{\mathsf{H}} X.
\]
Then $B,D$ are positive semidefinite, $S$ is Hermitian, and
\[
YY^{\mathsf{H}}-XX^{\mathsf{H}}=XH^{\mathsf{H}}+HX^{\mathsf{H}}+HH^{\mathsf{H}}.
\]
A direct expansion gives
\[
\begin{aligned}
\|YY^{\mathsf{H}}-XX^{\mathsf{H}}\|_F^2
=
\Tr\Bigl(
&B^2+4BS+2S^2+2DB
\Bigr).
\end{aligned}
\]
Let
\[
\eta:=2(\sqrt2-1)s^2.
\]
Then
\[
\begin{aligned}
\|YY^{\mathsf{H}}-XX^{\mathsf{H}}\|_F^2-\eta\|H\|_F^2
=
\Tr\Bigl((B+\sqrt2 S)^2+B\bigl((4-2\sqrt2)S+2D-\eta I\bigr)
\Bigr).
\end{aligned}
\]
The first trace term is nonnegative because $B+\sqrt2 S$ is Hermitian.  For the
second term, use $S=Y^{\mathsf{H}} X-D$ to obtain
\[
(4-2\sqrt2)S+2D-\eta I
=
(4-2\sqrt2)Y^H X+2(\sqrt2-1)D-\eta I.
\]
Since $Y^{\mathsf{H}} X$ is positive semidefinite and $D=X^{\mathsf{H}} X\succeq s^2 I$, the right-hand side is positive
semidefinite.  Therefore the second trace term is also nonnegative, and
\[
\|YY^{\mathsf{H}}-XX^{\mathsf{H}}\|_F^2
\geq
2(\sqrt2-1)s^2\|H\|_F^2.
\]
Since $\|H\|_F=\operatorname{dist}_F(Y,X)$ under the chosen alignment, the lemma follows.
\end{proof}

We now prove the theorem. 
\begin{proof}

Apply the lemma above with
\[
X=A_\ell(\rho_1^*)\equiv A_\ell,
\qquad
Y=A_\ell(\rho).
\]
Since $A_\ell A_\ell^{\mathsf{H}}=C_\ell[\Phi,\rho_1^*]$, we obtain
\[
\operatorname{dist}_F( A_\ell(\rho),A_\ell)
\leq
c_0\,
\frac{
\|A_\ell(\rho)(A_\ell(\rho))^{\mathsf{H}}-C_\ell[\Phi,\rho_1^*]\|_F
}{s_\ell}.
\]
By construction,
\[
A_\ell(\rho)(A_\ell(\rho))^{\mathsf{H}}
=
\mathcal{P}_{2\ell+1}^+(C_\ell[\Phi,\rho]).
\]
Because
\[
C_\ell[\Phi,\rho_1^*]=A_\ell A_\ell^{\mathsf{H}}
\]
is positive semidefinite and has rank at most $2\ell+1$, it is an admissible competitor in
the definition of $\mathcal{P}_{2\ell+1}^+(C_\ell[\Phi,\rho])$. Therefore
\[
\big\|\mathcal{P}_{2\ell+1}^+(C_\ell[\Phi,\rho])-C_\ell[\Phi,\rho]\big\|_F
\leq
\|C_\ell[\Phi,\rho_1^*]-C_\ell[\Phi,\rho]\|_F
=
\|\Delta_\ell(\rho)\|_F.
\]
By the triangle inequality,
\[
\begin{aligned}
\big\|A_\ell(\rho)(A_\ell(\rho))^{\mathsf{H}}-C_\ell[\Phi,\rho_1^*]\big\|_F
&=
\big\|\mathcal{P}_{2\ell+1}^+(C_\ell[\Phi,\rho])-C_\ell[\Phi,\rho_1^*]\big\|_F \\
&\leq
\big\|\mathcal{P}_{2\ell+1}^+(C_\ell[\Phi,\rho])-C_\ell[\Phi,\rho]\big\|_F
+
\|C_\ell[\Phi,\rho]-C_\ell[\Phi,\rho_1^*]\|_F \\
&\leq
2\|\Delta_\ell(\rho)\|_F.
\end{aligned}
\]
Substituting this into the factor bound gives
\[
\operatorname{dist}_F(A_\ell(\rho),A_\ell)
\leq
2c_0\,\frac{\|\Delta_\ell(\rho)\|_F}{s_\ell}.
\]
Under the conditions of Proposition~\ref{prop:Kam_mat_stability}, we further have
\[
\operatorname{dist}_F(A_\ell(\rho),A_\ell)
\le
\frac{2c_0}{s_\ell}\,
\Gamma_\ell(\Phi)\,
\beta_{2L}(\rho).
\]
    
\end{proof}

\end{document}